\documentclass[journal]{IEEEtran}
\usepackage{cite}

\ifCLASSINFOpdf
  \usepackage[pdftex]{graphicx}
\else
\fi
\usepackage{amsmath}
\usepackage{amssymb}
\usepackage{amsfonts}
\usepackage{enumitem}
\usepackage{dsfont}
\usepackage{xcolor}
\usepackage{algorithm, algorithmic}

\usepackage{url}

\DeclareMathOperator*{\argmax}{\arg\!\max}

\usepackage{amsthm}

\newtheorem{theorem}{Theorem}
\newtheorem{lemma}{Lemma}

\theoremstyle{remark}

\theoremstyle{definition}
\newtheorem{definition}{Definition}
\newtheorem{assumption}{Assumption}

\usepackage{makecell}
\usepackage{multirow}
\renewcommand{\arraystretch}{1.2}

\usepackage[caption=false, font=footnotesize, labelformat=simple]{subfig}

\begin{document}
%

\title{C-MASS: Combinatorial Mobility-Aware Sensor Scheduling for Collaborative Perception with Second-Order Topology Approximation}


%
%

\author{Yukuan~Jia,
        Yuxuan~Sun,~\IEEEmembership{Member,~IEEE,}
        Ruiqing Mao,
        Zhaojun~Nan,~\IEEEmembership{Member,~IEEE,} \\
        Sheng~Zhou,~\IEEEmembership{Senior Member,~IEEE,}
        and~Zhisheng~Niu,~\IEEEmembership{Fellow,~IEEE}
        
\thanks{Yukuan Jia, Ruiqing Mao, Zhaojun Nan, Sheng Zhou, and Zhisheng~Niu are with Department of Electronic Engineering, Tsinghua University, China. Emails: \{jyk20, mrq20\}@mails.tsinghua.edu.cn, \{nzj660224, sheng.zhou, niuzhs\}@tsinghua.edu.cn.}%
\thanks{Yuxuan~Sun is with School of Electronic and Information Engineering, Beijing Jiaotong University, China. Email: yxsun@bjtu.edu.cn}

}
\maketitle

\begin{abstract}
Collaborative Perception~(CP) has been a promising solution to address occlusions in the traffic environment by sharing sensor data among collaborative vehicles~(CoV) via vehicle-to-everything~(V2X) network.
With limited wireless bandwidth, CP necessitates task-oriented and receiver-aware sensor scheduling to prioritize important and complementary sensor data.
However, due to vehicular mobility, it is challenging and costly to obtain the up-to-date perception topology, i.e., whether a combination of CoVs can jointly detect an object.
In this paper, we propose a combinatorial mobility-aware sensor scheduling (C-MASS) framework for CP with minimal communication overhead. 
Specifically, detections are replayed with sensor data from individual CoVs and pairs of CoVs to maintain an empirical perception topology up to the second order, which approximately represents the complete perception topology.
A hybrid greedy algorithm is then proposed to solve a variant of the budgeted maximum coverage problem with a worst-case performance guarantee.
The C-MASS scheduling algorithm adapts the greedy algorithm by incorporating the topological uncertainty and the unexplored time of CoVs to balance exploration and exploitation, addressing the mobility challenge. 
Extensive numerical experiments demonstrate the near-optimality of the proposed C-MASS framework in both edge-assisted and distributed CP configurations. 
The weighted recall improvements over object-level CP are 5.8\% and 4.2\%, respectively.
Compared to distance-based and area-based greedy heuristics, the gaps to the offline optimal solutions are reduced by up to 75\% and 71\%, respectively.
\end{abstract}

\begin{IEEEkeywords}
collaborative perception, V2X network, sensor scheduling, maximum coverage, submodularity.
\end{IEEEkeywords}

%

\IEEEpeerreviewmaketitle

\section{Introduction}
Environmental perception, the fundamental component in automated driving, has rapid advancements in recent years~\cite{multi-modal}. 
However, onboard sensors, e.g., LiDAR and camera, provide visual information limited to its line of sight~(LoS), inherently unable to manage occlusions in complex traffic environments.
Collaborative Perception~(CP)~\cite{CP-survey-MITS, CP-survey-MNET1, CP-survey-MVT, CP-survey-MNET2} addresses this challenge by enabling connected automated vehicles~(CAV) and roadside units~(RSU) to exchange perception data via vehicle-to-everything~(V2X) communications~\cite{NR-V2X}. 
In the CP procedure, the sensor data from diverse perspectives are fused to achieve holistic environmental awareness, thus eliminating blind zones and discovering occluded and hard objects.

Depending on the availability of RSU, CP can operate in two configurations, as illustrated in Fig.~\ref{fig:illus}.
The first is edge-assisted CP~\cite{edge-sec, edgecooper, avp, lehan, hybrid-level}, where an RSU serves as the fusion center to observe a critical region, such as an intersection.
It aggregates sensor data from nearby collaborative vehicles (CoV) by vehicle-to-infrastructure~(V2I) communications, executes object detection, and broadcasts the results to vehicles via infrastructure-to-vehicle~(I2V) communications.
Since it is costly to deploy RSU, distributed CP~\cite{cpm, autocast, v2vnet, where2comm, xuemin} is the second choice outside coverage.
CoVs can exchange sensor data through direct vehicle-to-vehicle~(V2V) communications.
The current standard of CP is the collective perception message~(CPM)~\cite{cpm} that broadcasts a list of detected objects.
This object-level CP is prone to error when there is inevitably time asynchrony or inaccurate localization~\cite{v2vnet}.
Therefore, recent work considers transmitting LiDAR point clouds~\cite{autocast} or neural network-extracted features~\cite{v2vnet, where2comm}, exploiting and fusing richer information to improve the perception quality. 

\begin{figure}[t]
	\centering
	\subfloat[]{\includegraphics[width=0.71\linewidth]{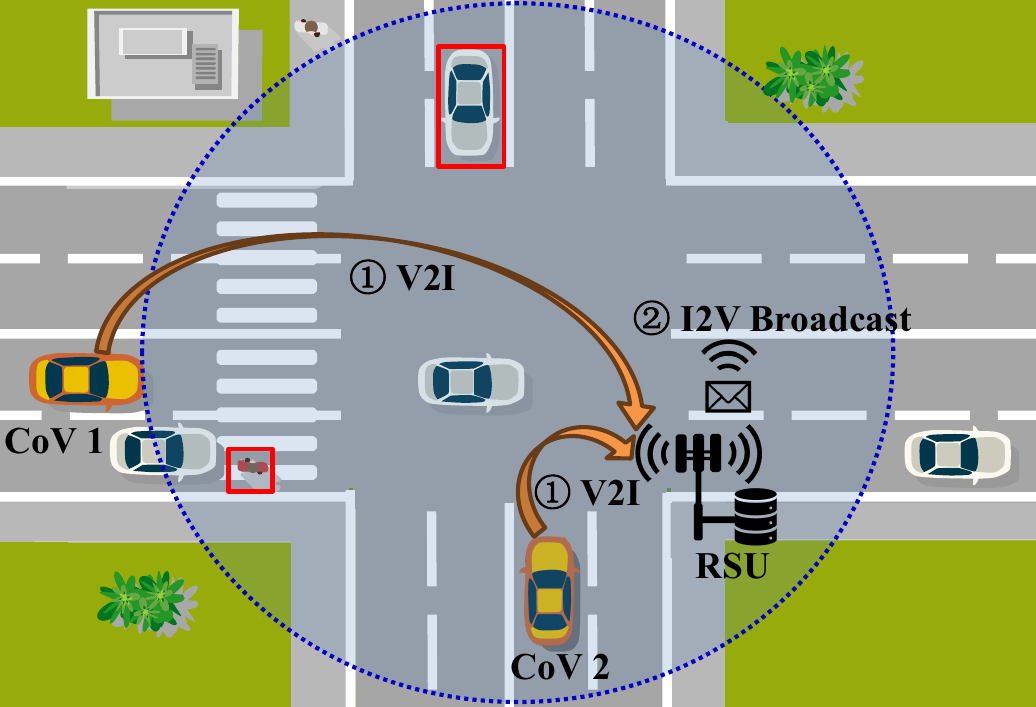}
		\label{subfig:illus-edge}}
	\hfill
	\subfloat[]{\includegraphics[width=0.71\linewidth]{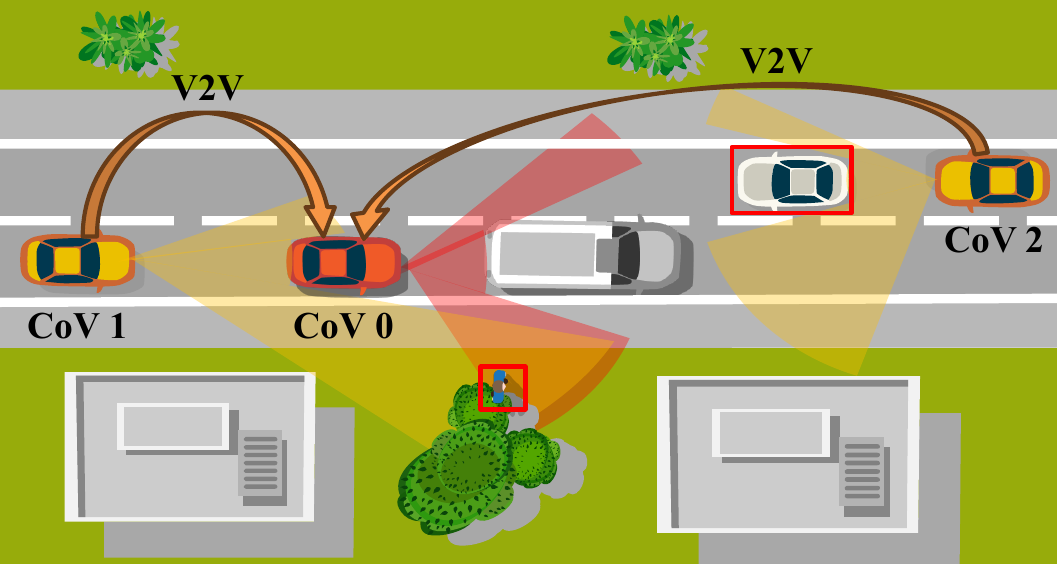}
		\label{subfig:illus-distr}}
	\caption{Illustrations of the two configurations of CP. (a) Edge-assisted CP with V2I/I2V communication. (b) Distributed CP with V2V communication.}
	\label{fig:illus}
\end{figure}

The allocated bandwidth of V2X communications is 20MHz in China and 30MHz in the US~\cite{spectrum}, which can hardly support real-time streaming of point clouds or features, not to mention coexistence with other V2X applications.
For bandwidth efficiency, data transmissions in CP need to be scheduled in a \emph{task-oriented} ~\cite{ours} and \emph{receiver-aware} way to prioritize critical and complementary sensor data over irrelevant or redundant data.
Given the set of CoVs available for CP, it is thus important to determine \emph{the subset of CoVs to communicate with} to optimize the utility of CP under a certain bandwidth constraint.

The unique challenge in effectively scheduling a CP system lies in the mobility of traffic participants such as vehicles and pedestrians~\cite{meet}.
Therefore, the \emph{perception topology}, defined as whether a combination of CoVs can jointly detect an object, is unknown and time-varying.
In fact, it is very difficult to predict the perception topology accurately, due to different sensor qualities, uncertain occlusion relationships, and the black-box nature of neural networks.
Besides, the number of combinations is exponential to the number of available CoVs. 
Existing CP frameworks rely on exchanging metadata messages, such as future trajectories~\cite{autocast}, object information~\cite{selective-comm}, or spatial confidence maps~\cite{where2comm}, to predict the perception topology or collaboration utility of CoVs individually.
This communication overhead occupies extra bandwidth and introduces latency, which impairs the overall performance~\cite{syncnet}. 
Additionally, due to the channel fading and the movement of CoVs, the wireless channel conditions are also time-varying, necessitating \emph{real-time} scheduling decisions and resource allocation.

In this paper, motivated by these challenges, we endeavor to introduce a combinatorial structure to the perception topology in feature-level CP. 
Supported by experimental observations, we assume that the perception topology of a set of CoVs can be approximately derived from the perception topologies of all individual CoVs and pairs of CoVs within the set, which is a \emph{second-order approximation}.
It is then feasible to calculate the utility of scheduling any CoV set with manageable complexity. 
Based on this assumption, we propose a combinatorial, mobility-aware sensor scheduling (C-MASS) framework for CP.
Specifically, the user, such as the RSU in centralized CP or a certain CoV in distributed CP, proactively requests and receives sensor data from a subset of nearby CoVs.
After data fusion and detection, the user pinpoints whether and how each CoV contributes to the detection of an object by replaying the detections.
Consequently, the user can efficiently schedule CoVs while balancing the trade-off between exploration and exploitation to maintain up-to-date knowledge of the perception topology.

\subsection{Related Work} 
\subsubsection{Collaborative Perception}
Depending on the type of transmitted data, there are three levels of CP~\cite{CP-survey-MITS}.
The current standard of object-level CP transmits lightweight CPMs containing labels and positions of detected object.
Analogous to hard decision, object-level CP cannot integrate the underlying rich information from multiple sensor data, and it has difficulty fusing discrepant results from different sources~\cite{association}.  
On the contrast, transmitting raw images or point clouds as in raw-level CP results in unnecessarily high data volume.

Feature-level CP strikes a balance by transmitting neural network-extracted features.
Many state-of-the-art CP models utilize bird-eye-view~(BEV) features~\cite{where2comm, cobevt}, which has a unified representation space for different perspectives and modalities, facilitating data fusion and compression.
For example, PointPillars~\cite{pp} encodes the LiDAR point clouds in each 2D BEV grid separately. 
CaDDN~\cite{caddn} first encodes images to the 3D space by depth prediction and then collapse to the 2D BEV space.
Fusing BEV features enjoys several benefits.
First, the fusion network can softly aggregate the local information across CoVs even with spatial misalignment and temporal asynchrony~\cite{coalign}.
Second, as reported in ref.~\cite{hybrid-level}, there are hard objects undetectable in any individual perspective but recognized when the information are fused. 
Last but not least, there are handy feature compression techniques such as background removal~\cite{autocast} and autoencoder~\cite{cobevt} to reduce the data size.

\subsubsection{Sensor Scheduling}
A body of literature focuses on scheduling a subset of sensors to maximize object coverage ~\cite{autocast, random-greedy} or area coverage~\cite{timely-coverage, david, rao} under communication constraint.
These NP-Hard combinatorial problems are usually accompanied with a nice submodular property that assumes diminishing returns, and can be effectively solved by greedy algorithms~\cite{approx-alg}. 
Spatial reasoning~\cite{autocast} and occupancy flow prediction~\cite{rao} methods based on ray tracing are used to predict LoS relationships.
However, modeling perception based on the LoS condition leads to an oversimplification that is too idealized for modern perception systems with imperfect factors such as sensor noise and poor lighting conditions.

The multi-agent learning paradigm is widely applied in collaborator scheduling in existing CP architectures~\cite{icc,offload,replication}.
When2Com~\cite{when2com} adopts a three-stage handshake mechanism, where query vectors are exchanged in prior to calculate matching scores between CoVs, which are then used to prune unnecessary collaborations.  
Furthermore, Where2Comm~\cite{where2comm} proposes transmitting spatial confidence maps alongside sensor data in multi-round collaboration, enabling sparsification compression of BEV features.
Select2Col~\cite{select2col} utilizes a graph neural network~(GNN) for collaborator scheduling.
In addition to introducing communication overhead, these methods rely on a global preset threshold or an implicit selection network, which fails to account for the time-varying and heterogeneous bandwidth costs caused by volatile wireless links.

\subsection{Contributions}
Our main contributions are summarized as follows:
\begin{enumerate}
    \item We present an analytical model of feature-level CP with a second-order approximation to the perception topology, supported by experimental observations. 
    Additionally, we fit a statistical detection model to CP datasets, facilitating large-scale evaluation of CP scheduling algorithms.
    \item We propose C-MASS, a pull-based, object-oriented CP framework that replays object detections locally to maintain and predict an empirical second-order perception topology, facilitating efficient scheduling of CP with minimal communication overhead.
    \item We propose a hybrid greedy algorithm to solve a variant of the budgeted maximum coverage problem that considers joint detections. The algorithm has low complexity and offers a theoretical performance guarantee. Furthermore, to balance exploration and exploitation, the C-MASS scheduling algorithm adapts the greedy approach by incorporating topological uncertainty and the unexplored time of CoVs.
    \item We conduct extensive numerical experiments in both edge-assisted and distributed CP configurations. The results demonstrate the near-optimal performance of the C-MASS framework and scheduling algorithm, outperforming distance-based and area-based scheduling algorithms.
\end{enumerate}

The remainder of this paper is organized as follows. 
Empirical studies on feature-level CP are conducted in Section~II.
Section III introduces the core modules of the C-MASS framework, and Section IV formulates the scheduling problem. 
Section V describes the proposed algorithm and conducts complexity and performance analyses.
Simulation results are provided in Section VI.
Finally, we conclude this paper with a summary in Section VII.

\section{Empirical Study on Feature-Level CP}
Towards characterizing feature-level CP with an analytical model, we conduct an empirical study on two popular open CP datasets in this section. 
The aim is to discover main properties of feature-level CP and develop an analytical model beyond the LoS rule that captures those properties. 
We focus on the LiDAR-based 3D object detection in this preliminary study.

\subsection{Datasets and Benchmarks}
We conduct experiments on two representative CP datasets, OPV2V~\cite{opv2v} and V2VReal~\cite{v2v4real}.
OPV2V is the first large-scale \emph{simulated} CP dataset for V2V cooperation.
It has over 70 scenarios collected from 8 towns.
In each scenario, there are 2 to 7 CoVs, which are mounted with omni-directional LiDARs to produce realistic point clouds.
V2V4Real is the first large-scale \emph{real-world} CP dataset for V2V cooperation.
Two experimental vehicles are used to collect the sensor data over distance of 410 kilometers.
Each vehicle is equipped with a Velodyne VLP-32 LiDAR and a GPS/IMU system.

The fundamental metrics of 3D object detection are recall and precision.
Recall is defined as the fraction of detected objects among all interested objects, while precision is defined as the fraction of true detections among all detections.
Precision and recall form a trade-off curve when the score threshold for positive detection varies. 
To reflect the overall capability of a detector, Average Precision~(AP) is defined as the integral of precision over recall.
As reported by the benchmarks of OPV2V and V2V4Real datasets~\cite{opv2v, v2v4real}, a state-of-the-art feature-level CP model CoBEVT~\cite{cobevt} achieves the best performance.
CoBEVT takes the encoded BEV features from multiple vehicles and fuses them using the transformer architecture.
Therefore, we adopt CoBEVT as the representative feature-level CP model for the following studies.

\begin{table}[ht]
\centering
\caption{3D object detection results on CP datasets}
\label{table:benchmark}
\renewcommand{\arraystretch}{1.3}
\resizebox{\columnwidth}{!}{
\begin{tabular}{lcccc}
\hline
\multirow{2}{*}{Method} & \multicolumn{2}{c}{OPV2V\cite{opv2v}}                  & \multicolumn{2}{c}{V2V4Real\cite{v2v4real}}            \\ \cline{2-5} 
                        & \multirow{2}{*}{\ \ AP (\%) \ }  & Recall (\%)  & \multirow{2}{*}{\ \ AP (\%) \ } & Recall (\%) \\
                        &                                  & @ Prec. 90\%    &                                  & @ Prec. 90\%    \\ \hline
Standalone              & 85.3                             & 84.1          & 59.5                             & 48.9          \\ \hline
Object-level CP         & 91.8                             & 91.3          & 79.4                             & 65.2          \\ \hline
CoBEVT \cite{cobevt}    & \textbf{96.9}                             & \textbf{96.5}          & \textbf{83.9}                             & \textbf{72.5}          \\ \hline
\end{tabular}}
\end{table}

In our reproduced experiment, the benchmark results on the test sets are shown in Table~\ref{table:benchmark}.
It is expected that CP, even at object level, significantly increases the AP by up to 20\%.
A notable improvement, about 5\% in AP, can be further achieved with the feature-level CP model CoBEVT.
By jointly considering the perceptual information from multiple perspectives, feature-level CP can not only discover more missed objects but also eliminate false positives.
We focus on a specific operating point where the required precision is 90\%, which means false positives should rarely happen. 
It can be observed that CoBEVT achieves 5.2\% and 7.3\% higher recall than the object-level CP in OPV2V and V2V4Real respectively, which are considerable improvements in 3D object detection that necessitate the feature-level CP.
We pick a frame from V2V4Real as an illustrative example.
As shown in Fig.~\ref{fig:lidar}, there are three objects that are only detectable with both sources of data using feature-level CP.
In addition, the positional accuracy of detections are also improved.

\begin{figure}[t]
	\centering
	\subfloat[]{\includegraphics[width=0.65\linewidth]{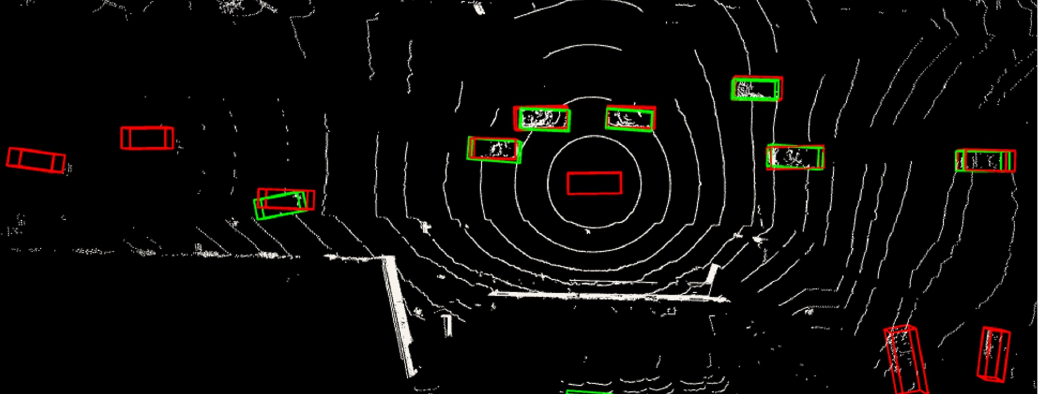}
		\label{subfig:lidar-ego}}
	\hfill
	\subfloat[]{\includegraphics[width=0.65\linewidth]{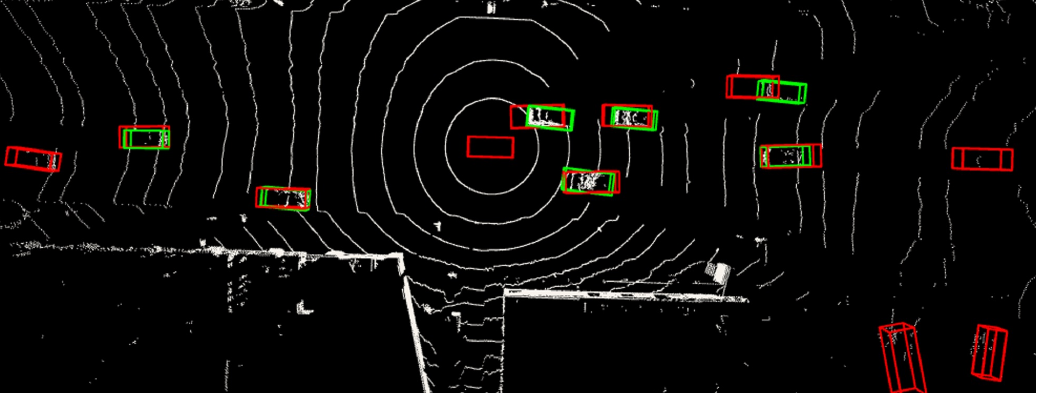}
		\label{subfig:lidar-aux}}
        \hfill
	\subfloat[]{\includegraphics[width=0.65\linewidth]{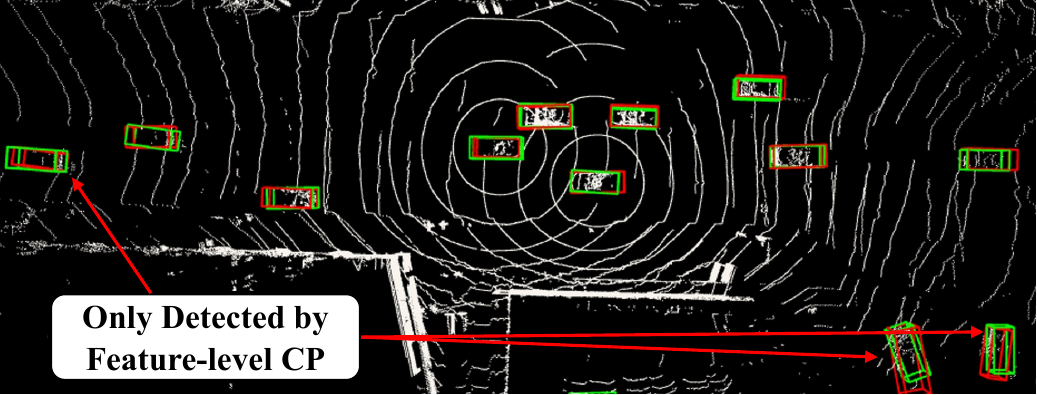}
		\label{subfig:lidar-merged}}        
	\caption{The LiDAR point clouds and detection results of a frame in V2V4Real. (a) The ego vehicle only. (b) The auxiliary vehicle only. (c) Feature-level CP. Red boxes: ground truth objects. Green boxes: objects detected by CoBEVT.}
	\label{fig:lidar}
\end{figure}

\subsection{Statistical Detection Model for Feature-level CP} \label{sect:stat-model}
A preliminary statistical study of 3D object detection on DOLPHINS dataset~\cite{dolphins} is conducted in our previous work~\cite{mass}.
Given a detection network $\Phi(\cdot)$, the conditional missed detection probability conforms closely to a negative power function of the number of scanned points $N(\mathcal{X}_n)$, i.e.,
\begin{align}
    \mathrm{Pr}(\Phi(\mathcal{X}_n)=0 | \Phi, N(\mathcal{X}_n)) &= (N(\mathcal{X}_n))^{-\lambda},
\end{align}
where $\Phi(\mathcal{X}_n)\in\{0,1\}$ denotes whether object $n$ is detected, and $\lambda$ is a scale parameter related to the capability of the detector $\Phi$ and the overall difficulty of the dataset. 
Define \emph{object difficulty} $D_{n}$ as the logarithm of the required number of scanned points for detection of object $n$ with detector $\Phi$.
Then, the statistics implies that the object difficulty follows an exponential distribution, i.e., $D_n \sim \mathrm{Exp}(\lambda)$. 
The difficulty of an object incorporates various influencing factors such as the unusual shape and the interactions with the background environment.
On the other hand, $\log{N(\mathcal{X}_n)}$ reflects the \emph{information amount} extracted from the sensing data $\mathcal{X}_n$.

However, there is currently no statistical detection model for feature-level CP, because it is challenging to create a simple and unified model with a variable number of collaborators.
In scenarios with at least three CoVs in OPV2V, we observe that only a negligible 0.06\% portion (32 out of 57,973) of true detections cannot be detected with any pair of two CoVs.
In other words, \emph{the vast majority of true detections originate from the sensing data of individual CoVs or the fused information of two certain CoVs}.
Therefore, we propose to approximate the CP result with multiple CoVs by the union of detected objects from all the CoV subsets containing no more than two CoVs, which is a \emph{second-order approximation}.
Note that the first-order approximation reduces to the object-level CP, which demonstrates a clear recall gap to the feature-level CP.

\begin{figure}[t]
	\centering
	\subfloat[]{\includegraphics[width=0.24\textwidth]{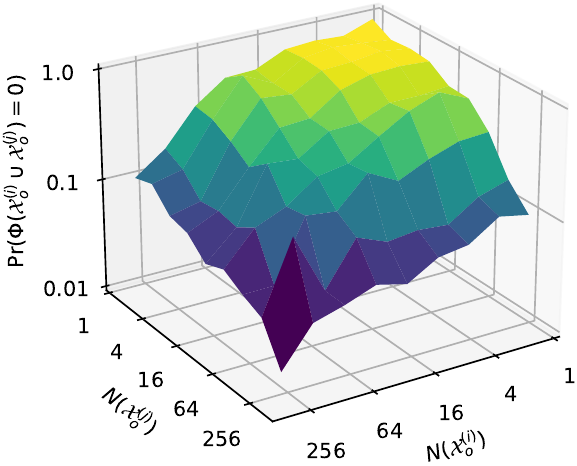}\label{subfig:v2v4real-stat}}
	\subfloat[]{\includegraphics[width=0.24\textwidth]{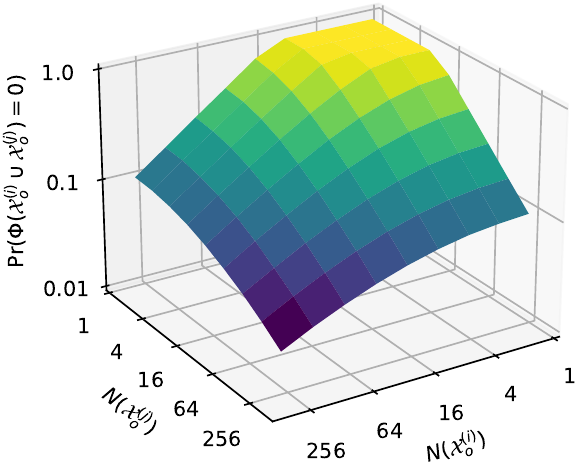}\label{subfig:v2v4real-stat-fit}}
	\caption{3D surface graphs of conditional missed detection probabilities given the number of scanned points at both perspectives. (a) Statistics from V2V4Real. (b) Fitted model for V2V4Real.}
	\label{fig:stat-fit}
\end{figure}


Furthermore, we perform statistics on the detection results across all pairs of CoVs.
Fig.~\ref{subfig:v2v4real-stat} presents the 3D surface graphs of the conditional missed detection probability using the real-world V2V4Real dataset, with similar results observed in the OPV2V dataset.
The figure shows that the missed detection probability decreases as the number of scanned points increases.
Notably, additional points from an auxiliary view are often more beneficial than those from the original view.
Inspired by this observation, we propose an empirical detection model based on the norm of the information amount from participating CoVs $\mathcal{V}$, i.e.,
\begin{align}
\quad \ &\left. \mathrm{Pr}\left(\Phi\left(\bigcup_{i\in\mathcal{V}}\mathcal{X}_{i,n}\right)=0\ \right| \Phi, \left\{N\left(\mathcal{X}_{i,n}\right)\right\}_{i\in\mathcal{V}} \right) \notag \\
= &\left. \mathrm{Pr}\left(\left|\left|\left\{\log{N\left(\mathcal{X}_{i,n}\right)}\right\}_{i\in\mathcal{V}}\right|\right|_p < D_n \right| \left\{N(\mathcal{X}_{i,n})\right\}_{i\in\mathcal{V}} \right),
\end{align}
where $\mathcal{X}_{i,n}$ denotes the LiDAR point clouds from CoV $i$ on object $n$, and $p$ is the norm order.
The object difficulty $D_n$ follows a shifted exponential distribution, with probability distribution function
\begin{align} \label{eq:shift-exp}
    f(d;\lambda,\mu) = 
    \begin{cases}
        \lambda \mathrm{e}^{-\lambda(d-\mu)}, &d \ge \mu, \\
        0, &d < \mu,
    \end{cases}
\end{align}
where $\lambda$ and $\mu$ are the scale and bias of the shifted exponential distribution. 

\begin{table}[ht]
\centering
\caption{Model Parameters Fitted on CP Datasets}
\label{table:fitparam}
\begin{tabular}{llll}
\hline
Dataset & Norm order $p$ & Diffic. scale $\lambda$ & Diffic. bias $\mu$ \\ \hline
V2V4Real~\cite{v2v4real}         & \multicolumn{1}{c}{2.3}            &   \multicolumn{1}{c}{2.1}                      &   \multicolumn{1}{c}{3.9}             \\ \hline
OPV2V~\cite{opv2v}            & \multicolumn{1}{c}{1.4}           &     \multicolumn{1}{c}{1.6}                    &  \multicolumn{1}{c}{0.9}            \\ \hline
\end{tabular}%
\end{table}

We fit the model parameters $p$, $\lambda$, and $\mu$ to the statistics on V2V4Real and OPV2V datasets.
Monte Carlo simulation is used to calculate the value of the fitted models for each grid point.
We adopt L1 loss on the logarithm of missed detection probability, because L1 loss is robust in the presence of noise in the data.
The optimal sets of model parameters is found by the brute-force search, as shown in Table~\ref{table:fitparam}.
The fitted missed detection probability of the V2V4Real is shown in Fig.~\ref{subfig:v2v4real-stat-fit}.
It can be seen that the fitted model approximates the overall trend of the missed detection probabilities.
The real-world dataset V2V4Real is more difficult than the simulated dataset OPV2V, because the sensing data is more noisy.
The norm order $p$ reflects the fusion efficiency of a CP model on a given dataset.
As a special case, $p=\infty$ corresponds to the object-level CP, which performs hard decisions.

In practice, the difficulty of an object at a certain position is inherent and unknown to the user. 
Therefore, it is impractical to directly regard the object difficulties as constraints for optimization.
In this work, the second-order approximation serves as a fundamental assumption of the system model.
The parameterized model is only used to efficiently conduct large-scale numerical experiments to evaluate the sensor scheduling algorithms in Sect.~\ref{sect:exp}.

\section{The C-MASS Framework} 

\begin{figure*}[ht]
    \centering
    \includegraphics[width=0.65\textwidth]{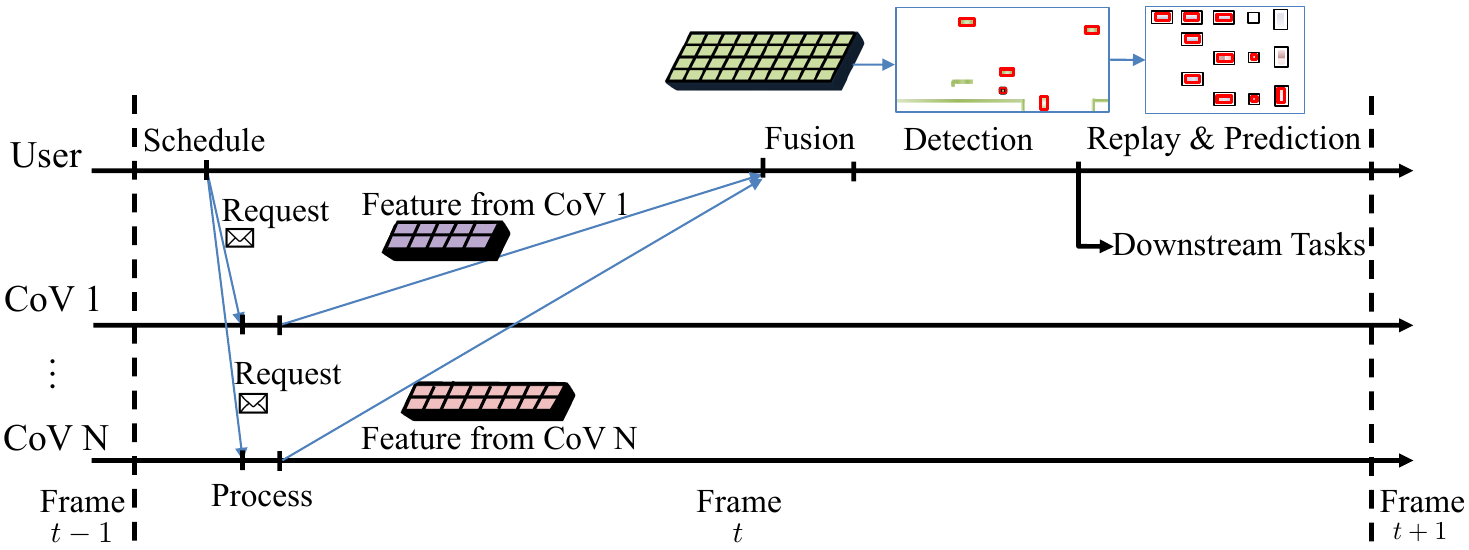}
    \caption{The workflow of the C-MASS framework.}
    \label{fig:pipeline}
\end{figure*}

Consider a discrete-time intelligent transportation system, where vehicles are driving on the road and pedestrians are walking on the sidewalks.
A specific portion of the vehicles are CoVs, which periodically beacon CAMs to broadcast their motion status and declare their sensor sharing availability.
At time frame $t$, the CoVs available for sensor sharing are denoted by the set $\mathcal{V}_t$. 
The unconnected vehicles and the pedestrians are interested objects, denoted by $\mathcal{O}_t$, which need to be continuously detected by CP.
Depending on the configuration, the user could be a specific CoV in distributed CP, or the edge server served as the fusion center in edge-assisted CP.
Both $\mathcal{V}_t$ and $\mathcal{O}_t$ are time-varying because the vehicles and pedestrians move in and out of the communication range and the interest range of the user. 
We consider a pull-based CP paradigm, because the sensor data can be transmitted on demand and tailored to the need of the user.
The complete workflow of the C-MASS framework is illustrated in Fig.~\ref{fig:pipeline}.

\subsection{Communication Module}
The total available bandwidth for CP service of the user is given by $W_t$, which can be time-varying according to the network congestion status.
Assume the channel gain $c_i(t)$ between the user and any CoV $i\in\mathcal{V}_t$ is measurable by receiving the CAMs. 
We adopt a pull-based communication procedure introduced as follows.
The user first sends request messages to a subset of CoVs $\mathcal{A}_t \subseteq \mathcal{V}_t$, specifying its interested region and the allocated bandwidth $W_i(t)$.
Upon receiving a request, the CoV compresses and transmits its sensing data $\mathcal{X}_{i}^{(t)}$ accordingly. 
By Shannon's formula, the achievable transmission rate is given by
\begin{align} \label{eq:rate}
    r_i(t) = W_i(t)\log_2\left(1+\frac{P c_i(t)}{n_0 W_i(t)}\right),
\end{align}
where $P$ is the transmitting power, $n_0$ is the power spectrum density of the additive white Gaussian noise.
The rate constraint is given by
\begin{align} \label{eq:rate-constr}
     D\left(\mathcal{X}_{i}^{(t)}\right) \le r_i(t) \Delta t,
\end{align}
where $D(\mathcal{X}_{i}^{(t)})$ denotes the size of the requested BEV feature from CoV $i$, and $\Delta t$ is the length of communication in a frame, relating to the maximum tolerable delay of CP.
From (\ref{eq:rate}) and (\ref{eq:rate-constr}), the user can find a \emph{minimum required allocated bandwidth} $B_i(t)$ for lossless transmission of sensing data $\mathcal{X}_{i}^{(t)}$, by solving
\begin{align}
    B_i(t) \log_2\left(1+\frac{P c_i(t)}{n_0 B_i(t)}\right) = D\left(\mathcal{X}_{i}^{(t)}\right) / \Delta t.
\end{align}

\subsection{Detection Module}
In time frame $t$, the user receives the sensing data $\mathcal{X}_{i}^{(t)}$ from the scheduled CoVs $\mathcal{A}_t$ and fuses the data in the feature space.
The fused feature map is fed through the backbone neural network and the detector head $\Phi(\cdot)$ to obtain the detection result of CP.
The detection result of CP with $\mathcal{A}_t$ is defined by
\begin{align}
    \mathcal{P}_t(\mathcal{A}_t) = \left\{ n \in \mathcal{O}_t \left| \Phi\left(\bigcup_{i\in\mathcal{A}_t}\mathcal{X}_{i,n}^{(t)}\right) = 1  \right\}, \right.
\end{align}
where $\mathcal{X}_{i,n}^{(t)}$ denotes the sensing data of object $n$ from CoV $i$ at time frame $t$.
\begin{definition}
    The \emph{perception topology} of a CP system with CoVs $\mathcal{V}_t$ at time frame $t$ is defined by $\{\mathcal{P}_t(\mathcal{S}), \forall \mathcal{S} \subseteq \mathcal{V}_t\}$.
\end{definition}

The complete perception topology includes the detection results of any subset of available CoVs. 
Note that the user has no accurate prior knowledge on the perception topology.
Even if perfect knowledge is available, a memory and computation complexity of $O(2^{|\mathcal{V}_t|})$ is required to store and compare all the subsets.
Based on the observation in Sect.~\ref{sect:stat-model}, we make the following assumption.

\begin{assumption} \label{ass1}
\emph{(Second-order Approximation)} Any detectable object by CP with CoVs $\mathcal{S}$ is detectable by an essential subset consisting of an individual CoV or a pair of CoVs. 
Formally, 
\begin{align}
    \mathcal{P}_t(\mathcal{S}) &= \left(\bigcup_{i \in \mathcal{S}} \mathcal{P}_t^{(1)}(i)\right) \cup \left(\bigcup_{\substack{i,j \in \mathcal{S}, i \neq j}} \mathcal{P}_t^{(2)}(i,j)\right),
\end{align}
where the \emph{first-order perception topology} of a CoV $i \in \mathcal{S}$ is defined by 
\begin{align}
    \mathcal{P}_t^{(1)}(i) = \left\{ n \in \mathcal{O}_t \left| \Phi\left(\mathcal{X}_{i,n}^{(t)}\right) = 1  \right\}\right.,
\end{align}
and the \emph{second-order perception topology} of two distinct CoVs $i,j \in \mathcal{A}_t$ is defined by 
\begin{align}
    \mathcal{P}_t^{(2)}(i,j) &= \left\{ n \in \mathcal{O}_t \left| \Phi\left(\mathcal{X}_{i,n}^{(t)} \cup \mathcal{X}_{j,n}^{(t)}\right) = 1  \right\}\right. \notag \\
    &\quad \ \setminus \mathcal{P}_t^{(1)}(i) \setminus \mathcal{P}_t^{(1)}(j).
\end{align}
\end{assumption}
This assumption introduces a combinatorial structure for the perception topology.
It is then possible to derive the complete perception topology from the 
first-order perception topology $\{\mathcal{P}_t^{(1)}(i), \forall i\in\mathcal{V}_t\}$ and the second-order perception topology $\{\mathcal{P}_t^{(2)}(i,j), \forall i,j\in\mathcal{V}_t, i \neq j\}$, reducing the memory complexity to $O(|\mathcal{V}_t|^2)$.

\subsection{Replay Module}
\begin{figure*}[ht]
    \centering
    \includegraphics[width=0.78\textwidth]{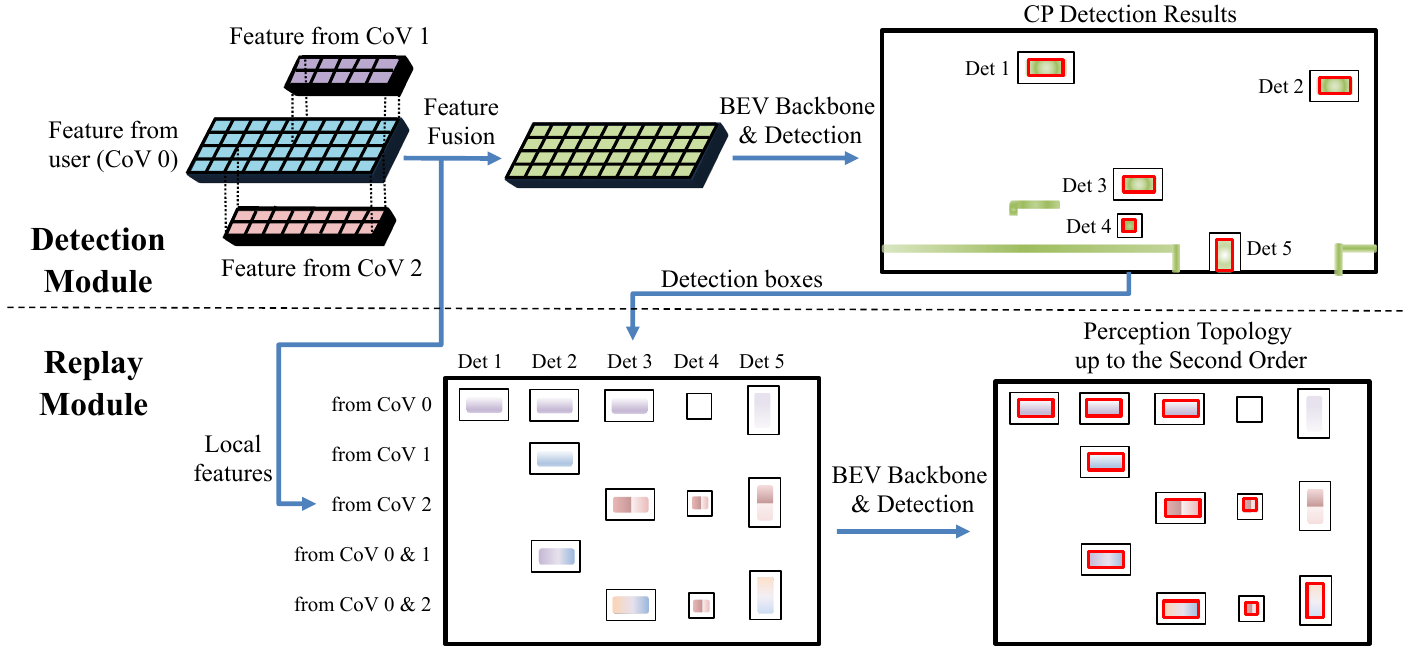}
    \caption{Illustration of the replay module.}
    \label{fig:replay}
\end{figure*}

Since the perception topology is unknown and changes gradually, following the idea of \emph{learning while scheduling}~\cite{mass}, we propose a replay module to feedback the perception topology after CP.
The aim of the replay module is to identify the perception topology associated with the scheduled CoVs.

As shown in Fig.~\ref{fig:replay}, the local features around the detection boxes are retrieved from the BEV feature maps. 
The local features from any individual CoVs and fused from any pair of CoVs, if exists, form a new feature map for replay.
The replay feature map is fed into the detection head, which outputs the perception topology up to the second order. 
While there can be $(|\mathcal{A}_t|-1)|\mathcal{A}_t|/2$ local features for replaying a single detection, the size of local features is very small compared to the size of a complete feature map.
Therefore, it is computationally efficient with spatial-wise and possibly batch-wise parallelization.
In real-world deployment, CP models such as CoBEVT can achieve an inference speed of over 50 fps when collaborating with 5 CoVs~\cite{cobevt}.
Therefore, it is feasible to conduct the replay before the next frame of CP, which is usually required at no more than 10 fps.

With the replay module, the user can maintain an empirical perception topology up to the second order from the latest observation, i.e.,
\begin{align}
    \Hat{\mathcal{P}}_t^{(1)}(i) = \mathcal{P}_{\tau_1}^{(1)}(i),&\quad \forall i \in \mathcal{V}_t, \label{eq:replay-first} \\
    \Hat{\mathcal{P}}_t^{(2)}(i,j) = \mathcal{P}_{\tau_2}^{(2)}(i,j),&\quad \forall i,j \in \mathcal{V}_t,\ i \neq j, \label{eq:replay-second}
\end{align}
where $\tau_1 = \arg\max_{i \in \mathcal{A}_s,s<t} s$, $\tau_2 = \arg\max_{i,j \in \mathcal{A}_s,s<t} s$ are the latest scheduled times.
For fresh knowledge of the perception topology, the user needs to acquire the sensor data from less beneficial CoVs once in a while.
Otherwise, the user would schedule with inaccurate empirical perception topology, which impairs the overall CP performance.

\subsection{Prediction Module}
For frame $t+1$, the empirical perception topology can be partially corrected by motion prediction and spatial reasoning.
Specifically, the user can construct a local environment map with the predicted positions of the tracked objects.
The LoS condition between each CoV and object can be further deduced from the map by ray tracing method.
While a LoS condition does not guarantee a detection, it is the necessary condition for detection.
Therefore, we can refine the empirical perception topology by updating
\begin{align}
    \Hat{\mathcal{P}}_{t+1}^{(1)}(i) &\gets \Hat{\mathcal{P}}_{t+1}^{(1)}(i) \cap \Hat{\mathcal{L}}_{t+1}(i), \label{eq:confine-first}\\
    \Hat{\mathcal{P}}_{t+1}^{(2)}(i,j) &\gets \Hat{\mathcal{P}}_{t+1}^{(2)}(i,j) \cap \Hat{\mathcal{L}}_{t+1}(i) \cap \Hat{\mathcal{L}}_{t+1}(j), \label{eq:confine-second}
\end{align}
where $\Hat{\mathcal{L}}_{t+1}(i) \subseteq \mathcal{O}_t$ denotes the set of objects that are predicted to have a LoS path with CoV $i$.
Although there are other unpredictable factors except occlusions, the refined empirical perception topology becomes more conservative and accurate.

Moreover, when the LoS path between a CoV and an object emerges, it is more likely that the perception topology changes, which requires exploration.
Define the topological uncertainty $\mathcal{U}_t(i)$ as the set of objects that are predicted to be additionally detected by CoV~$i$.
It is updated by
\begin{align}
    \mathcal{U}_{t+1}(i) = 
    \begin{cases}
        \left(\mathcal{L}^C_t(i) \cap \Hat{\mathcal{L}}_{t+1}(i)\right) \cup \mathcal{U}_{t}(i), &\quad i \in \mathcal{A}_t, \\
        \mathcal{L}^C_t(i) \cap \Hat{\mathcal{L}}_{t+1}(i), &\quad i \notin \mathcal{A}_t,
    \end{cases} \label{eq:update-topo-unc}
\end{align}
where $\mathcal{L}^C_t(i) = \mathcal{O}_t \setminus \mathcal{L}_t(i)$ denotes the set of objects without a LoS path with CoV $i$ at frame $t$.
The CoVs with topological uncertainty are encouraged to be scheduled for exploration, depending on the specific implementation of the algorithm.

\section{Problem Formulation}
To account for the heterogeneous influence on downstream tasks, we adopt a \emph{weighted recall} metric, defined as
\begin{align} \label{eq:def:actual}
    \sum_{t=1}^T \sum_{n \in \mathcal{P}_t (\mathcal{A}_t)} w_n(t) \bigg/ \sum_{t=1}^T \sum_{n \in \mathcal{O}_t} w_n(t),
\end{align}
where $w_n(t)$ is the importance weight of object $n$ at frame $t$, depending on the type and the state of the user.
Particularly, if $w_n(t)=1$ for any object at any time, the weighted recall is same as the conventional recall metric.
Define the \emph{CP utility} of scheduling $\mathcal{A}_t$ at frame $t$ by
\begin{align}
    g_t(\mathcal{A}_t)=\sum_{n \in \mathcal{P}_t (\mathcal{A}_t)} w_n(t).
\end{align}
The objective is to maximize the CP utility over a trip of $T$ frames, by scheduling optimal sets of CoVs for CP under the bandwidth constraint, i.e.,
\begin{align} \label{formula:problem}
    \max_{\mathcal{A}_t \subseteq \mathcal{V}_t} &\quad \sum_{t=1}^T g_t(\mathcal{A}_t), \\
    \mathrm{s.t.} 
    &\quad \sum_{i\in\mathcal{A}_t} B_i(t) \le W_t,\ \ \forall t=1,2,\dots,T.
\end{align}

This is a variant of the budgeted maximum coverage (BMC) problem~\cite{bmc}.
In the classic BMC problem, each set can \emph{independently} cover certain elements, and there are possibly overlaps.
The goal is to maximize the total weights of the covered elements by selecting a collection of sets under a given budget.
If the first-order approximation of the perception topology was adopted, the sensor scheduling problem would be equivalent to the BMC problem.
The BMC problem, even with unit set costs, is NP-Hard by a polynomial reduction from the set cover problem.
It has a \emph{submodularity} property~\cite{approx-alg} defined below.

\begin{definition}
    A set function $f: 2^{|X|}\to\mathbb{R}$ is \emph{submodular} if 
    \begin{align}  \label{eq:def-submodular}
        f(S\cup\{i\}) - f(S) \ge f(T\cup\{i\}) - f(T),
    \end{align}
    for $\forall S \subseteq T \subset X$ and $i \in X \setminus T$.
\end{definition}

The submodularity property requires that there is a diminishing marginal return for adding a sensor when the already scheduled sensor collection is larger. 
It is natural to solve with a greedy algorithm that sequentially picks the sensor with the highest ratio of uncovered element weights to cost.
Leveraging the submodularity property, the greedy algorithm is proved to achieve a $(1-1/e)$-approximation ratio, which is the best possible approximation algorithm~\cite{bmc}. 
However, with the second-order approximation to perception topology, $g_t(\mathcal{S})$ is not submodular.
Although there is a generalized definition of submodularity ratio~\cite{submodular, non-submodular}, it has zero value in our setting which leads to no performance guarantee.
Therefore, applying greedy algorithm directly to the actual CP utility can perform very badly.

\begin{figure}[t]
	\centering
	\subfloat[]{\includegraphics[width=0.42\linewidth]{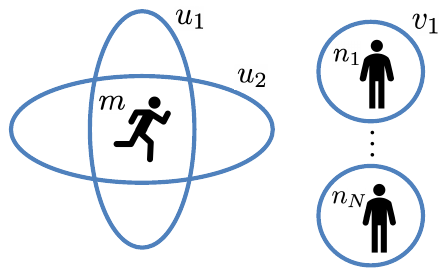}
		\label{subfig:case1}}
	\hfill
	\subfloat[]{\includegraphics[width=0.53\linewidth]{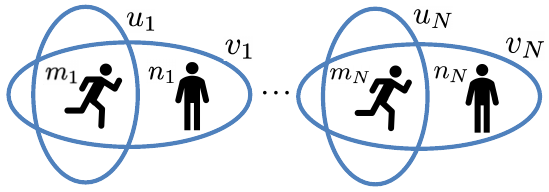}
		\label{subfig:case2}}
	\caption{Perception topologies in two artificial cases. A running pedestrian requires two sensors for detection, while a standing pedestrian requires one sensor for detection.}
	\label{fig:cases}
\end{figure}

\textbf{Example 1}: The object set is $\mathcal{O}=\{m, n_1, \dots, n_N\}$ with weights $w_m = 1$, $w_{n_1} = \dots = w_{n_N} = \epsilon > 0$. 
The CoV set is $\mathcal{V} = \{u_1, u_2, v_1, \dots, v_N\}$ with unit cost. 
The perception topology is illustrated in Fig.~\ref{subfig:case1}, where object $m$ requires joint detection from $u_1$ and $u_2$.
With a cost budget of $B \le N$, the greedy algorithm that maximizes the \emph{actual utility} sequentially selects $v_1, \dots, v_B$, receiving a $\epsilon B$ perception gain.
On the other hand, the optimal gain is greater than $1$ when $u_1, u_2$ are scheduled together.
Therefore, the worst-case approximation ratio of the greedy algorithm approaches $0$ as $\epsilon \to 0$, because it cannot take joint detections into consideration.

To mitigate this problem, we define the \emph{pending utility} as
\begin{align} \label{eq:def-pending}
    &g^+_t (\mathcal{S}) = g_t(\mathcal{S}) \notag \\
    &+ \sum_{n \in \mathcal{P}_t^C(\mathcal{S})} w_n(t) \max_{\substack{i\in\mathcal{S} \\ j\in \mathcal{S}^C}}\mathds{1}\left\{n\in\mathcal{P}^{(2)}_t(i,j)\right\} \tfrac{B_i(t)}{B_i(t)+B_j(t)},
\end{align}
where $\mathcal{P}^C_t (\mathcal{S}) = \mathcal{O}_t \setminus \mathcal{P}_t (\mathcal{S})$ and $\mathcal{S}^C = \mathcal{V}_t \setminus \mathcal{S}$ are complement of sets.
The object weight from a joint detection of CoVs $i$ and $j$ is evenly split according to their bandwidth costs.
In this way, fair credits are given to the preliminary CoV although actual utility is not immediately incurred.
The pending gain has the nice property of submodularity.

\begin{lemma} \label{lem:submodular}
    The pending utility $g_t^+(\mathcal{S})$ is submodular.
\end{lemma}
\begin{proof}
    See Appendix \ref{app:submodular}.
\end{proof}
One may attempt to maximize the pending utility in a greedy manner, but Example 2 closes the door. 

\textbf{Example 2}: The object set is $\mathcal{O}=\{m_1, \dots, m_N, n_1$, $ \dots, n_N\}$ with weights $w_{m_1} = \dots = w_{m_N} = 1, w_{n_1} = \dots = w_{n_N} = \epsilon > 0$. 
The CoV set is $\mathcal{V} = \{u_1, \dots, u_N, v_1, \dots, v_N\}$ with unit cost. 
The perception topology is illustrated in Fig.~\ref{subfig:case2}, where object $m_i$ requires joint detection from $u_i$ and $v_i$.
With a cost budget of $B \le N$, the greedy algorithm maximizing the pending utility sequentially selects $v_1, \dots, v_B$, with an actual utility of $\epsilon B$.
However, when $u_i$ and $v_i$ are scheduled together, the optimal utility is at least $\lfloor B/2 \rfloor$.
The worst-case approximation ratio approaches $0$ again as $\epsilon \to 0$.
The problem is that the pending utility may never get transformed into the actual CP utility within the cost budget.

Furthermore, since the perception topology $\mathcal{P}_t$ is unavailable and dynamic, the user needs to learn and maintain a fresh empirical perception topology $\Hat{\mathcal{P}}_t$.
A less helpful CoV should be scheduled occasionally to gain knowledge about its current perception topology, which is termed \emph{exploration}.
The user then optimizes the sensor scheduling among CoVs with the up-to-date, mostly accurate perception topology, which is termed \emph{exploitation}.
Therefore, the algorithm should also balance exploration and exploitation in the sensor scheduling process.

\section{Algorithms}
In this section, we first address the non-submodularity of the actual CP utility by proposing a low-complexity hybrid greedy algorithm for the one-shot scheduling problem. 
Then we apply modifications to address the exploration-exploitation trade-off over time.

\subsection{The Hybrid Greedy Algorithm}
For now, let us assume the empirical perception topology is accurate, i.e., $\Hat{\mathcal{P}}_t = \mathcal{P}_t$.
Then the decisions become decoupled, and the time index $t$ is omitted for simplicity.

To capture the joint detections while being pragmatic, we propose a \emph{hybrid greedy algorithm} as shown in Algorithm~\ref{alg:greedy}.
Define a \emph{hybrid perception utility}, which is a weighted average of the actual utility and the pending utility, expressed by 
\begin{align}
    h(S) &= \lambda g^+(\mathcal{S}) + (1-\lambda) g(\mathcal{S}).
\end{align}
The parameter $\lambda\in(0,1)$ adjusts the tendency to schedule a CoV for future joint detection.
In each round, the algorithm sequentially schedules the CoV that maximizes the utility-to-cost ratio.
The algorithm maintains the current detection levels of each object by a vector $\mathbf{d} \in [0,1]^{|O|}$. 
The $n$-th element of $\mathbf{d}$ is given by
\begin{align} \label{eq:def-level}
    \mathbf{d}_n = 
    \begin{cases}
        1, &n \in \mathcal{P}(\mathcal{A}), \\
        \max\limits_{\substack{i\in\mathcal{A} \\ j\in \mathcal{A}^C}}\mathds{1}\left\{n\in\mathcal{P}^{(2)}_t(i,j)\right\} \tfrac{B_i}{B_i+B_j}, &n \notin \mathcal{P}(\mathcal{A}).
    \end{cases}
\end{align}
According to the definitions (\ref{eq:def:actual}), (\ref{eq:def-pending}), the actual utility and the pending utility can be calculated by 
\begin{align}
    g(\mathcal{A}) &= \sum_{n\in\mathcal{O}} w_n \lfloor \mathbf{d}_n \rfloor, \label{eq:calc-actual}\\
    g^+(\mathcal{A}) &= \sum_{n\in\mathcal{O}} w_n \mathbf{d}_n, \label{eq:calc-pending}
\end{align}
For efficient computation of $\mathbf{d}$, we also introduce an auxiliary matrix $\mathbf{P}\in[0,1]^{|V|\times|O|}$.
As initialized in Line~3, each row of $\mathbf{P}$ represents the detection levels of objects of a CoV, incorporating both the first-order and the second-order perception topology.
In Line~12, after a CoV $v$ is scheduled, the detection levels of other CoVs are improved according to the second-order topology $\mathcal{P}^{(2)}(\cdot,v)$. 
By maintaining $\mathbf{d}$ and $\mathbf{P}$, the algorithm can efficiently calculate the \emph{marginal} utilities, defined by
\begin{align}
    g(i|\mathcal{A}) &= g(\mathcal{A} \cup \{i\}) - g(\mathcal{A}), \label{eq:marginal-actual} \\
    g^+(i|\mathcal{A}) &= g^+(\mathcal{A} \cup \{i\}) - g^+(\mathcal{A}). \label{eq:marginal-pending}
\end{align}

\begin{figure}[t]
    \renewcommand{\algorithmicrequire}{\textbf{Input:}}
    \renewcommand{\algorithmicensure}{\textbf{Output:}}
    \begin{algorithm}[H]
    \begin{algorithmic}[1]
        \REQUIRE $\mathcal{V}, \mathcal{P}^{(1)}(i), \mathcal{P}^{(2)}(i,j), w_n, B_i, W$
        \ENSURE $\mathcal{A}$
        \STATE $\mathcal{A} \gets \emptyset$
        \STATE Initialize the current detection level vector $\mathbf{d} = \mathbf{0}^{1\times|O|}$.
        \STATE Initialize the detection level matrix $\mathbf{P}\in[0,1]^{|V|\times|O|}$ by
        \begin{align} \label{eq:init-matrix}
            \mathbf{P}_{i,n} = 
            \begin{cases}
                1, &n \in \mathcal{P}^{(1)}(i), \\
                \max\limits_{j\in \mathcal{V}}\mathds{1}\left\{n\in\mathcal{P}^{(2)}(i,j)\right\} \tfrac{B_i}{B_i+B_j}, &n \notin \mathcal{P}^{(1)}(i),
            \end{cases}
        \end{align}
        \WHILE{$\min_{i \in \mathcal{V} \setminus \mathcal{A}} B_i \le W$}
            \FORALL{$i \in \mathcal{V} \setminus \mathcal{A}$}
                \STATE Calculate the \emph{marginal hybrid} perception utility by
                \begin{align}
                    g(i|\mathcal{A}) &= \sum_{n\in\mathcal{O}} w_n \max\{\lfloor \mathbf{P}_{i,n} \rfloor - \lfloor\mathbf{d}_n\rfloor,\ 0\} \label{eq:calc-marginal-immediate}, \\
                    g^+(i|\mathcal{A}) &= \sum_{n\in\mathcal{O}} w_n \max\{\mathbf{P}_{i,n} - \mathbf{d}_n,\ 0\} \label{eq:calc-marginal-pending}, \\
                    h(i|\mathcal{A}) &= \lambda g^+(i|\mathcal{A}) + (1-\lambda) g(i|\mathcal{A}).
                \end{align}
            \ENDFOR
            \STATE Schedule the CoV that maximizes utility-to-cost ratio
            \begin{align}
                v = \argmax_{i \in \mathcal{V} \setminus \mathcal{A}} \mathds{1}\left\{ B_i \le W \right\} \frac{h(i|\mathcal{A})}{B_i}.
            \end{align}
            \STATE $\mathcal{A} \gets \mathcal{A} \cup \{v\}$.
            \STATE $W \gets W - B_{v}$.
            \STATE Update the current detection level vector
            \begin{equation} \label{eq:update-level}
                \mathbf{d}_n \gets \max\{ \mathbf{d}_n,\ \mathbf{P}_{v,n}\},
            \end{equation}
            \STATE Update the detection level matrix
            \begin{align} \label{eq:update-matrix}
                    \mathbf{P}_{i,n} \gets 
                    \begin{cases}
                        1, &\quad n \in \mathcal{P}^{(2)} (i,v), \\
                        \mathbf{P}_{i,n}, &\quad n \notin \mathcal{P}^{(2)} (i,v).
                    \end{cases}
            \end{align}
        \ENDWHILE
    \end{algorithmic}
    \caption{The Hybrid Greedy Algorithm}
    \label{alg:greedy}
    \end{algorithm}
\end{figure}

\begin{lemma} \label{lem:alg1-calc}
    When running Algorithm~\ref{alg:greedy}, 
    \begin{enumerate}
        \item[a)] The current detection levels $\mathbf{d}$ equal to the definition (\ref{eq:def-level}).
        \item[b)] The marginal utilities can be calculated by (\ref{eq:calc-marginal-immediate}) and (\ref{eq:calc-marginal-pending}).
    \end{enumerate}
\end{lemma}
\begin{proof} 
    See Appendix~\ref{app:alg1-calc}.
\end{proof}
Since $\mathbf{d}$ is non-decreasing, Lemma \ref{lem:alg1-calc} also states that the utilities are non-decreasing.

\subsubsection{Complexity Analysis}
The space complexity is dominated by the second-order perception topology $\mathcal{P}^{(2)}(i,j)$. 
For each pair of CoVs $i,j\in\mathcal{V}$, the perception topology can be stored as a binary vector with length $|\mathcal{O}|$, and the total space complexity is $O(|\mathcal{V}|^2|\mathcal{O}|)$.

For the time complexity, the initialization of the detection level matrix $\mathbf{P}$ takes $O(|\mathcal{V}|)$ for each element because it needs to check the second-order perception topology.
The complexity of initialization is thereby $O(|\mathcal{V}|^2|\mathcal{O}|)$.
In each round, the user first calculates the marginal hybrid perception utility of each candidate CoV at a complexity of $O(|\mathcal{O}|)$, resulting in a computation complexity of $O(|\mathcal{V}||\mathcal{O}|)$.
The comparison takes $O(|\mathcal{V}|)$ and the update of $\mathbf{P}$ takes $O(|\mathcal{V}||\mathcal{O}|)$.
Therefore, over $O(|\mathcal{V}|)$ rounds, the total computation complexity is efficiently restricted to $O(|\mathcal{V}|^2|\mathcal{O}|)$.
Note that our proposed hybrid greedy algorithm matches the computation complexity of the vanilla greedy algorithm even with the first-order approximation, which also takes $O(|\mathcal{V}|^2|\mathcal{O}|)$. 

\subsubsection{Performance Analysis}
We consider a special case where the bandwidth costs of CoVs are uniform, and the bandwidth budget allows scheduling $N$ CoVs.

Define the maximum collaboration degree by 
\begin{align}
    C = \max_{i\in\mathcal{V}}\sum_{j\in\mathcal{V}} \mathds{1}\left\{ \mathcal{P}^{(2)}(i,j) \neq \emptyset \right\}.
\end{align}
Specifically, the problem is submodular when $C=0$.
We set the weight parameter $\lambda=1/(C+1)$, depending on the maximum collaboration degree.

\begin{lemma} \label{lem:local}
    Denote the set of scheduled CoVs after the $n$-th scheduling round by $\mathcal{A}_n$, $n \in \{0,1, \cdots,N-2\}$, Algorithm~\ref{alg:greedy} satisfies either
    \begin{align} \label{eq:one-step}
        g(\mathcal{A}_{n+1}) - g(\mathcal{A}_{n}) \ge \frac{\gamma}{N}(g(OPT) - g(\mathcal{A}_{n})),
    \end{align}
    or 
    \begin{align} \label{eq:two-step}
        g(\mathcal{A}_{n+2}) - g(\mathcal{A}_{n}) \ge \frac{2\gamma}{N}(g(OPT) - g(\mathcal{A}_{n})),
    \end{align}
    where 
    \begin{align} \label{eq:gamma}
        \gamma = \begin{cases}
            1, &\quad C_\mathrm{max}=0, \\
            1/6, &\quad C_\mathrm{max}=1, \\
            1/(6C+2), &\quad C_\mathrm{max}\ge2.
        \end{cases}
    \end{align}
\end{lemma}
\begin{proof}
    See Appendix~\ref{app:local}.
\end{proof}

Even without submodularity, Lemma~\ref{lem:local} states that Algorithm~\ref{alg:greedy} guarantees a growth of a positive factor of the optimality gap in one or two steps.
Based on this lemma, we can further derive a worst-case performance guarantee for Algorithm~\ref{alg:greedy}.

\begin{theorem} \label{thm}
    Algorithm \ref{alg:greedy} achieves an approximation ratio of $1-(1-\gamma/N)^{N-1}$ for the scheduling problem, i.e.,
    \begin{align}
        g(\mathcal{A}_\mathrm{greedy}) &\ge \left[1-\left(1-\frac{\gamma}{N}\right)^{N-1}\right] g(\mathcal{A}^*) \\
        & \stackrel{N\to\infty}{\relbar\joinrel\relbar\joinrel\longrightarrow} \left(1-e^{-\gamma}\right) g(\mathcal{A}^*),
    \end{align}
    where $\gamma$ is a constant determined by $C_\mathrm{max}$ in (\ref{eq:gamma}) .
\end{theorem}
\begin{proof} 
    See Appendix~\ref{app:thm}.
\end{proof}

Theorem~\ref{thm} provides a non-trivial approximation ratio for the hybrid greedy algorithm. 
While the proposed algorithm admits $\gamma \to 0$ as $C\to\infty$, it can be shown that, with a finite look-ahead step $L$, any greedy algorithm based on a weighted average of the actual utility and the pending utility has zero approximation ratio as $C\to\infty$.
We can see this by two examples.

\textbf{Example 3}: There are $N$ groups of unit-weight objects and unit-cost CoVs. In each group, there are $C$ objects, each of which is detectable jointly by a common CoV and a particular CoV that observes this object only.
If $\lambda > 2/C$, the hybrid greedy algorithm with any finite look-ahead step $L$ selects the $N$ common CoVs from distinct groups, because each of the common CoV incurs a hybrid utility greater than $1$.
Eventually, the resulting actual utility is 0.

\textbf{Example 4}: There are two groups, each of which has $C$ unit-cost CoVs.
In the first group, the CoVs are sub-grouped by size $L$.
Within each subgroup, any pair of CoVs jointly detects an object of weight $\sqrt{C}$.
In the second group of size $C$, any pair of CoVs jointly detects an object of weight $1$.
Within a budget of $C$, if $\lambda \le 2/C$, the $L$-step hybrid greedy algorithm schedules a subgroup from the first group once in a round, receiving $O(C\sqrt{C)}$ perception utility.
However, the optimal solution is to schedule the second group, which results in a $O(C^2)$ perception utility.
Therefore, the approximation ratio approaches $0$ as $C\to\infty$.

\subsection{Balancing Exploration and Exploitation}
Due to the mobility of vehicles and pedestrians, the perception topology is gradually changing over time.
Therefore, the user needs to learn the perception topology from historical data and balance between exploration and exploitation.
In ref.~\cite{mass}, we modeled the collaboration utility of each CoV as an independent Gaussian random walk with reflecting boundaries.
Then the user determines an upper confidence bound (UCB) of the collaboration utility for each CoV, expressed by $\Tilde{g}_i(t) = \Hat{g}_i(t) + \beta \sqrt{t-\tau_i}$.
The size of the confidence bound is proportional to the square root of its unexplored time since the CoV is last scheduled at $\tau_i$.
We have proved that scheduling the CoV with the best UCB of utility achieves near-optimal performance, which means the regret matches the lower bound up to a logarithmic factor.
Intuitively, the optimism towards collaboration utility can also be interpreted as a reward for fresh knowledge of the perception topology.

\begin{figure}[!t]
    \renewcommand{\algorithmicrequire}{\textbf{Parameters:}}
    \begin{algorithm}[H]
    \begin{algorithmic}[1]
    	\REQUIRE $\alpha, \beta$ 
    	\FOR {$t=1,\cdots,T$}
    	    \STATE Determine the available CoV set $\mathcal{V}_t$.
    	    \FOR {any CoV $v\in\mathcal{V}_t$ that has never been scheduled}
                    \STATE $\mathcal{A}_t \gets \mathcal{A}_t \cup \{v\}$.
                    \STATE $W \gets W - B_{v}$.
    	    \ENDFOR
            \STATE Execute Algorithm~\ref{alg:greedy} with $\Hat{\mathcal{P}}_t^{(1)}, \Hat{\mathcal{P}}_t^{(2)}$, replacing the marginal utilities $g(i|\mathcal{A}_t)$ and $g^+(i|\mathcal{A}_t)$ with $\Tilde{g}_t(i|\mathcal{A}_t)$ and $\Tilde{g}^+_t(i|\mathcal{A}_t)$ calculated by (\ref{eq:explore-marginal-actual}) and (\ref{eq:explore-marginal-pending}).
            \STATE Send CP requests to CoVs $\mathcal{A}_t$ for sensor data.
            \STATE Receive and fuse the sensor data for detection.
            \STATE Replay detections to update the perception topology $\Hat{\mathcal{P}}_{t+1}^{(1)}, \Hat{\mathcal{P}}_{t+1}^{(2)}$ by (\ref{eq:replay-first}) and (\ref{eq:replay-second}).
            \STATE Predict the LoS relationship $\Hat{\mathcal{L}}_{t+1}(i)$, refine $\Hat{\mathcal{P}}_{t+1}^{(1)}$, $\Hat{\mathcal{P}}_{t+1}^{(2)}$ by (\ref{eq:confine-first}) and (\ref{eq:confine-second}), and update $\mathcal{U}_{t+1}(i)$ by (\ref{eq:update-topo-unc}).
        \ENDFOR
    \end{algorithmic}
    \caption{The C-MASS Scheduling Algorithm}
    \label{alg:cmass}
    \end{algorithm}
\end{figure}

Extending the MASS framework~\cite{mass} to scheduling a combination of CoVs, we propose the C-MASS scheduling algorithm, as described in Algorithm~\ref{alg:cmass}.
When a CoV available for sensor sharing comes into the communication range of the user, it is immediately scheduled once to gain knowledge of its perception topology.
With the remaining bandwidth, Algorithm~\ref{alg:greedy} is executed with the empirical perception topology $\Hat{\mathcal{P}}_t^{(1)}$ and $\Hat{\mathcal{P}}_t^{(2)}$.
We propose to add two extra terms to the marginal CP utilities, i.e.,
\begin{align}
    \Tilde{g}_t(i|\mathcal{A}) = \Hat{g}_t(i|\mathcal{A}) + \alpha \sum_{n \in \mathcal{U}_t(i)} w_n + \beta \sqrt{t-\tau_i}, \label{eq:explore-marginal-actual}\\
    \Tilde{g}^+_t(i|\mathcal{A}) = \Hat{g}^+_t(i|\mathcal{A}) + \alpha \sum_{n \in \mathcal{U}_t(i)} w_n + \beta \sqrt{t-\tau_i}, \label{eq:explore-marginal-pending}
\end{align}
where $\alpha$ and $\beta$ are weighting constants that balance exploration and exploitation.
The first extra term reflects the topological uncertainty from the prediction module, and the second extra term is the UCB term swelling with the unexplored time.
Note that these two terms do not harm the execution and properties of Algorithm~\ref{alg:greedy}, because they act like the utility from the first-order perception topology when calculating the marginal utilities.
Lines 8-9 are the main process of CP, performing the data exchange, data fusion, and object detection. 
Finally, by replay and prediction, Lines 10-11 prepares the empirical perception topology $\Hat{\mathcal{P}}_{t+1}^{(1)}$ and $\Hat{\mathcal{P}}_{t+1}^{(2)}$ and the topological uncertainty $\mathcal{U}_{t+1}$ for the next frame.
The complexity of the complete C-MASS framework remains $O(|\mathcal{V}|^2|\mathcal{O}|)$, with no more than $|\mathcal{V}|^2|\mathcal{O}|$ detection replays and topology refinement.

\section{Numerical Experiments} \label{sect:exp}
In this section, we evaluate the proposed C-MASS framework with numerical experiments, using simulated mobility traces and the fitted statistical detection model of feature-level CP in Sect.~\ref{sect:stat-model}.

\subsection{Simulation Setup}
First, we generate the street map and the mobility traces of vehicles and pedestrians using the microscopic traffic simulator SUMO~\cite{sumo}.
As shown in Fig.~\ref{fig:map}, the map includes an area of $4\times4$ urban blocks with a side length of 200m.
The urban blocks are surrounded by bidirectional two-lane streets for vehicles and sidewalks for pedestrians.
A certain percentage of vehicles are CoVs, which is termed the market penetration ratio (MPR).
The CoVs are equipped with an omni-directional LiDAR on the top and ready to share their sensor data.
The positions of CoVs are accessible via the broadcast of CAMs.
The rest of the vehicles and pedestrians are objects that should be constantly perceived.
In the simulated environment, there are occlusions caused by the buildings on the block corners and the other vehicles on the road. 

\begin{figure}
	\centering
	\includegraphics[width=0.72\linewidth]{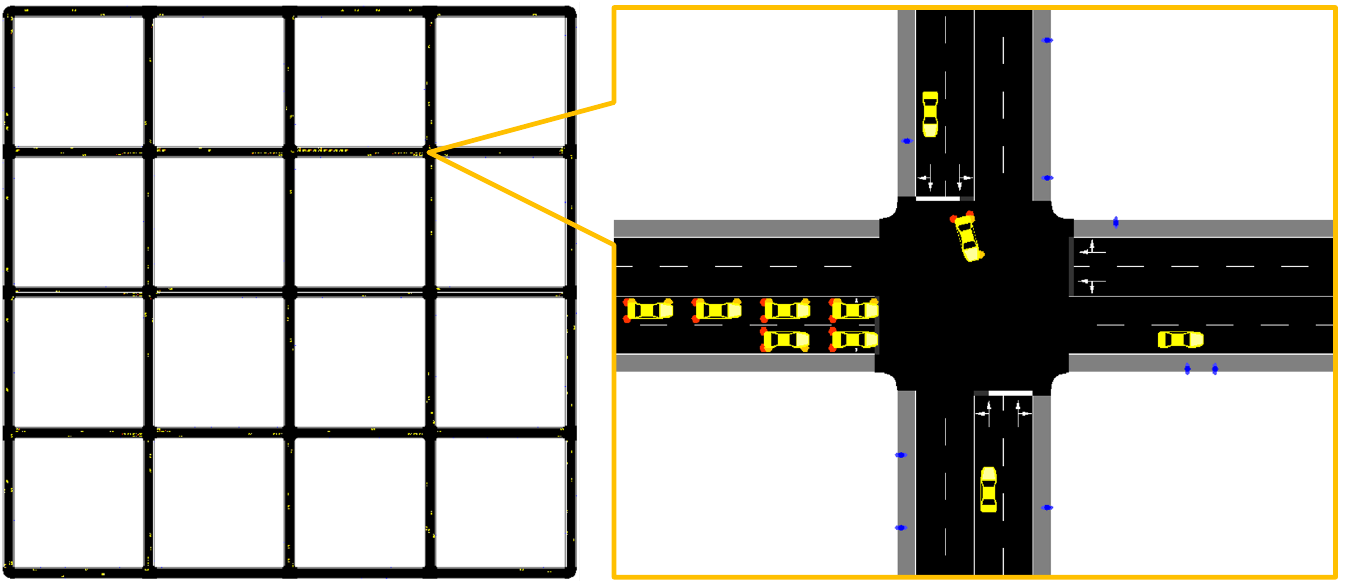}
	\caption{The street map built for SUMO simulation. The vehicles on the streets are in yellow, and the pedestrians on the sidewalks are in blue.}
	\label{fig:map}
\end{figure}

\begin{figure}
	\centering
	\subfloat[]{\includegraphics[width=0.443\linewidth]{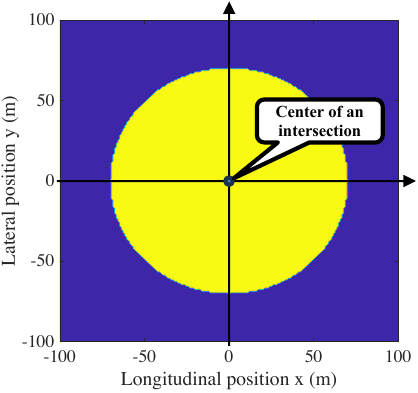}
		\label{subfig:interest-edge}}
	\hfil
	\subfloat[]{\includegraphics[width=0.49\linewidth]{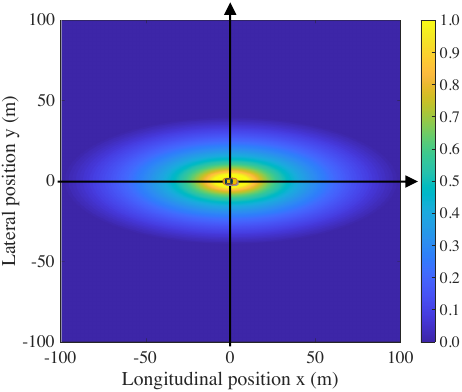}
		\label{subfig:interest-distr}}
	\caption{The interest range of the user in two different CP configurations. The color represents the importance weight of an object at the position. (a) Edge-assisted CP. (b) Distributed CP.}
	\label{fig:interest}
\end{figure}

We investigate two CP configurations.
One is edge-assisted CP, where the user is an edge server in charge of perceiving the area around one of the intersection and broadcast the perception results to the CoVs.
We focus on the data collection and fusion for process.
The interest range is a circle of radius $70\mathrm{m}$, and all objects $n$ within the interest range are equally weighted, i.e., $w_n(t) = 1$.
The other is distributed CP, where the user is one of the CoVs that perceives its local environment with the help of CP.
The interest range is a rectangular-shaped area $[\pm100\mathrm{m}, \pm40\mathrm{m}]$ centered at the user position, with the longer side towards the orienting direction. 
The importance weight of an object is calculated by
\begin{align}
    w_n(t) = \min\left\{\max\left\{-\log_{10}\sqrt{(x/100)^2+(y/40)^2}, 0\right\}, 1\right\},
\end{align}
where $x,y$ are relative longitudinal and lateral positions of the object.
The interest range in both configurations are summarized in Fig.~\ref{fig:interest}.

\begin{table}[!t]
\centering
\caption{Simulation Parameters}
\begin{tabular}{|ll|}
\hline 
\multicolumn{1}{|l|}{\textbf{Parameters}}                   & \textbf{Values}           \\ \hline\hline
\multicolumn{1}{|l|}{Frame Length}              & 0.1s                            \\ \hline
\multicolumn{1}{|l|}{Number of Vehicles}              & 200                               \\ \hline
\multicolumn{1}{|l|}{MPR of CoV}                   & 50\%                                    \\ \hline
\multicolumn{1}{|l|}{Speed Limit of Vehicles}              & 50km/h                               \\ \hline
\multicolumn{1}{|l|}{Arrival Rate of Pedestrians}            & 0.02 persons/s                 \\ \hline
\multicolumn{1}{|l|}{Speed of Pedestrians}              & 1.2m/s                 \\ \hline
\multicolumn{2}{|c|}{\textbf{Sensor-Related}}                                               \\ \hline
\multicolumn{1}{|l|}{\# of Lasers}                       & 32                       \\ \hline
\multicolumn{1}{|l|}{Vertical Field-of-view}             & 40$^{\circ}$                              \\ \hline
\multicolumn{1}{|l|}{Angular Resolution}                 & 0.1$^{\circ}$                              \\ \hline
\multicolumn{1}{|l|}{Maximum Range}                     & 100m                               \\ \hline
\multicolumn{1}{|l|}{Height of Objects}                 & 1.7m                               \\ \hline
\multicolumn{2}{|c|}{\textbf{Communication-related}} \\ \hline
\multicolumn{1}{|l|}{BEV Feature Size}         & 0.20MB for size 200m$\times$80m                            \\ \hline
\multicolumn{1}{|l|}{Max. Comm. Distance}                 & 150m    \\ \hline

\multicolumn{1}{|l|}{Carrier Frequency}                 & 5.9GHz                           \\ \hline
\multicolumn{1}{|l|}{Transmission Power}              & 23dBm                                \\ \hline
\multicolumn{1}{|l|}{Noise Power Spectral Density}     & -174dBm/Hz                              \\ \hline
\multicolumn{1}{|l|}{Receiver Noise Figure}            & 9dB                                 \\ \hline
\multicolumn{1}{|l|}{Shadowing Fading Std. Dev.}       & 3dB (LOS, NLOSv), 4dB (NLOS)         \\ \hline
\multicolumn{1}{|l|}{Vehicle Blockage Loss}            & max\{0, $\mathcal{N}(5,4)$\} dB     \\ \hline
\end{tabular}
\label{tab:param}
\end{table}

The simulation parameters are listed in Table~\ref{tab:param}.
We adopt the same parameters as the real-world setting of V2V4Real~\cite{v2v4real}, and simulate the laser scanning of a Velodyne VLP-32 LiDAR by the ray tracing method.
The detection model of V2V4Real in Sect.~\ref{sect:stat-model} is adopted, which captures important characteristics of CP including the long-tail difficulty of objects and the time-continuity between frames.
Every object is assigned a random difficulty by (\ref{eq:shift-exp}), specifying the minimum information amount required for perception.
For V2X communication, we adopt the V2X sidelink channel model in 3GPP TR 37.885~\cite{3gpp37885}, which classifies link conditions by LOS, NLOS, and NLOSv. 
Other than the conventional former two types, the NLOSv channel is caused by vehicular blockages, bringing extra attenuation for each blocking vehicle on the basis of the LOS channel.
Moreover, we consider fast fading with i.i.d. Rician distribution for LOS, NLOSv links, and Rayleigh distribution for NLOS links.
In the C-MASS framework, the prediction of the future object positions is implemented by a simple extrapolation, using the positions of the latest two detections.
We simulate $T=10,000$ time frames for both CP configurations.

\subsection{Simulation Results}
We compare the proposed C-MASS framework with three baseline scheduling algorithms for feature-level CP.
\textbf{Closest First} sequentially adds the closest CoV to the user into the scheduling set.
Closer CoVs generally enjoy better communication conditions and covers more interested areas.
\textbf{Greedy Area Coverage} greedily schedules the CoV with the maximum ratio of additional area coverage to bandwidth cost.
It approximates the behaviors of EdgeCooper~\cite{edgecooper} in edge-assisted CP configuration, and RAO~\cite{rao} in decentralized CP configuration.
\textbf{C-MASS (1st-order Topo.)} is a lightweight version of the C-MASS algorithm that only considers independent detections, i.e., the first-order perception topology, for utility calculation.
For reference, we also include object-level CP.
\textbf{CPM} is the current standard of CP, where each CoV independently detects objects and broadcast the results. 
Since it is communication-efficient, we assume the user always receives CPMs from all CoVs.
\textbf{Offline Optimal} solution is also provided by enumerating all feasible combinations of CoVs, which serves as the upper bound.

\begin{figure}[t]
	\centering
	\subfloat[]{\includegraphics[width=0.70\linewidth]{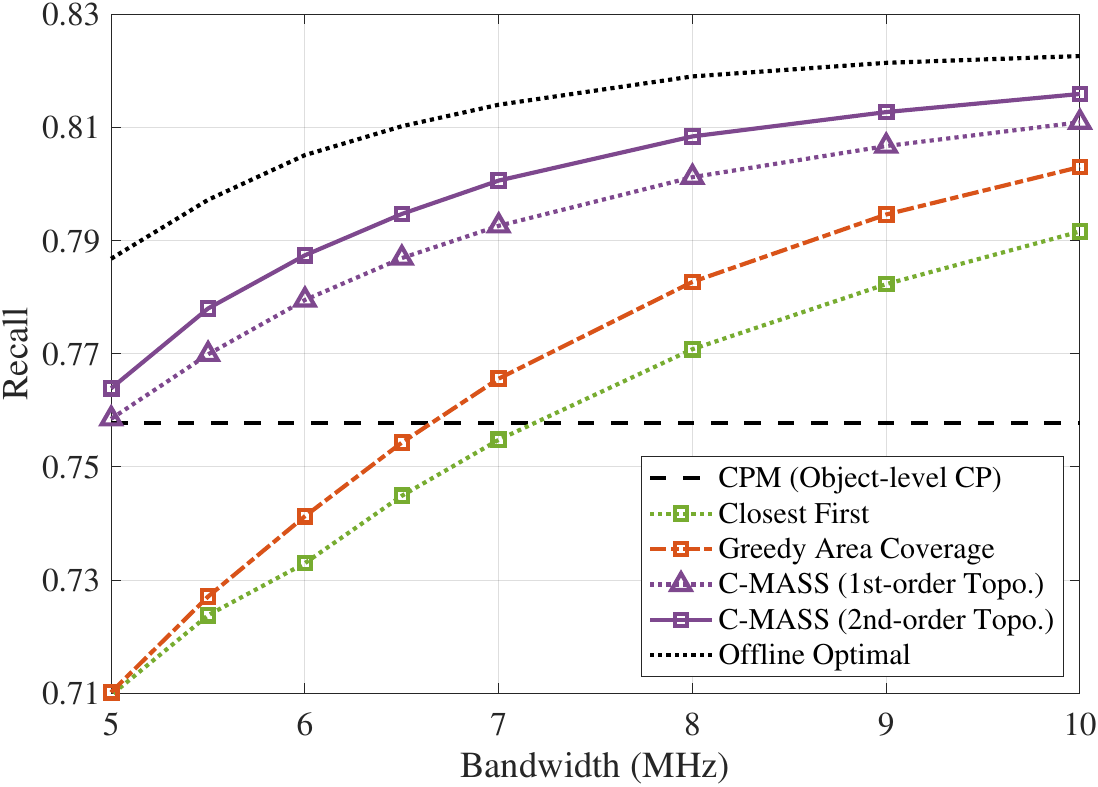}
		\label{subfig:compare-edge}}
	\hfill
	\subfloat[]{\includegraphics[width=0.70\linewidth]{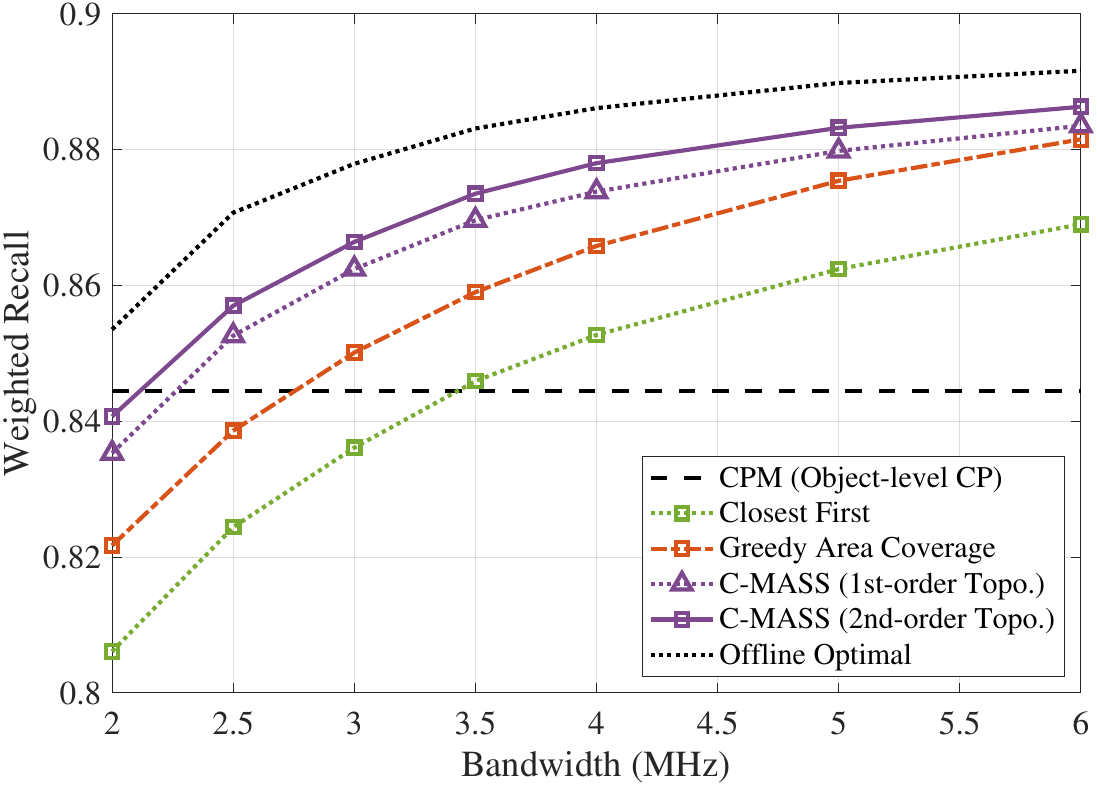}
		\label{subfig:compare-distr}}
	\caption{Performance comparison of different scheduling methods. (a) Edge-assisted CP. (b) Distributed CP.}
	\label{fig:compare}
\end{figure}

As shown in Fig.~\ref{fig:compare}, the proposed C-MASS scheduling algorithm achieves the best and near-optimal perception performance. 
Compared with the CPM standard, C-MASS achieves a significant 5.8\% higher recall and a 4.2\% higher weighted recall in the two configurations respectively.
This huge gain necessitates the feature-level CP as the next-generation CP system.
On the contrary, the performance loss of C-MASS to the offline optimal solutions is only 1.6\% and 0.9\% on average, which is mainly due to the inevitable exploration cost of the less beneficial CoVs.
Compared with Closest First and Greedy Area Coverage, the performance loss to the Offline Optimal is reduced by 74.7\%, 71.0\% in the edge-assisted configuration, and 74.0\%, 56.8\% in the distributed configuration on average, suggesting the advantage of object-oriented scheduling.
The weighted recall improvements over the first-order perception topology are 0.7\% and 0.4\% in two configurations, which are worthy gains given that the computational complexity remains unchanged.

In Fig.~\ref{fig:MPR}, the perception performance of C-MASS with varying MPR is presented.
With enough bandwidth, 5MHz for the edge-assisted configuration and 2.5MHz for the distributed configuration, C-MASS uniformly outperforms CPM standard at any MPR value.
Moreover, increasing the MPR of CoVs from 25\% to about 75\% dramatically increases the performance of CP, because there are more diverse perspectives to collaborate with.
At high MPR values, the relatively small gap above 75\% leaves a room for hybrid CoVs and human-driven vehicles in the future intelligent transportation system.

\begin{figure}[t]
	\centering
	\subfloat[]{\includegraphics[width=0.485\linewidth]{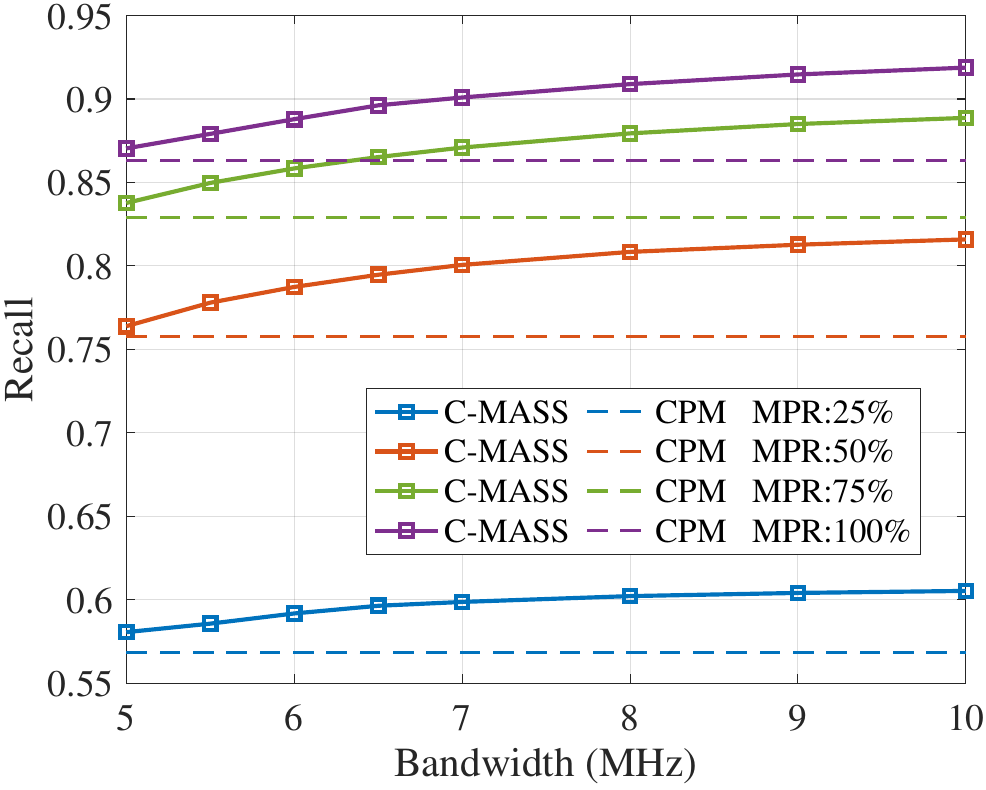}
		\label{subfig:MPR-edge}}
	\hfill
	\subfloat[]{\includegraphics[width=0.485\linewidth]{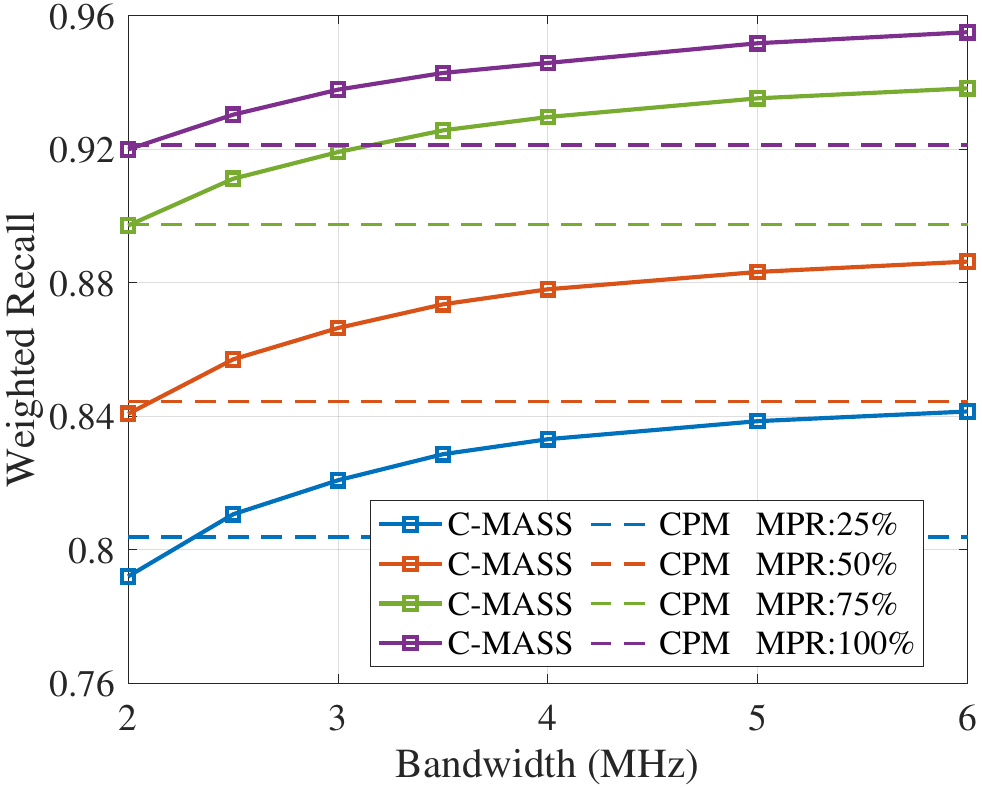}
		\label{subfig:MPR-distr}}
	\caption{Performance of C-MASS framework at different MPR values. (a) Edge-assisted CP. (b) Distributed CP.}
	\label{fig:MPR}
\end{figure}


\begin{table}[t]
\centering
\caption{Ablation Studies on Algorithm Features}
\begin{tabular}{ccc|c|c}
\hline
\multicolumn{3}{c|}{Algorithm Features}                                                                       & \multicolumn{2}{c}{CP Configurations} \\ \hline
\multicolumn{1}{l|}{UCB}          & \multicolumn{1}{l|}{Topo. Unc.} & \multicolumn{1}{l|}{Topo. Ref.} & Edge-assisted                                & Distributed                    \\ \hline
\multicolumn{1}{c|}{$\times$}     & \multicolumn{1}{c|}{$\times$}              & $\times$                                & 87.24                                & 80.23                              \\
\multicolumn{1}{c|}{$\checkmark$} & \multicolumn{1}{c|}{$\times$}              & $\times$                                & 87.54(+0.30)                          & 80.50(+0.27)                       \\
\multicolumn{1}{c|}{$\times$} & \multicolumn{1}{c|}{$\checkmark$}              & $\times$                                & 87.54(+0.30)                          & 80.52(+0.29)                       \\
\multicolumn{1}{c|}{$\checkmark$} & \multicolumn{1}{c|}{$\checkmark$}          & $\times$                                & 87.57(+0.33)                          & 80.55(+0.32)                       \\
\multicolumn{1}{c|}{$\checkmark$} & \multicolumn{1}{c|}{$\checkmark$}          & $\checkmark$                            & \textbf{87.80(+0.56)}                 & \textbf{80.84(+0.61)}              \\ \hline
\end{tabular}
\label{tab:ablation}
\end{table}

We have conducted ablation studies on the algorithmic features of C-MASS, as shown in Table~\ref{tab:ablation}.
Both the UCB and the topological uncertainty play the role of encouraging exploration, each resulting in an approximately 0.3\% weighted recall improvement, and a modest extra gain when applied together.
The topology refinement mechanism offers the remaining 0.3\% improvement by a best-effort correction of the empirical perception topology.
Furthermore, we find that the algorithm is robust to the values of exploration parameters $\alpha$ and $\beta$, as long as they are positive.
When sweeping $\alpha, \beta$ in the logarithmic space $[10^{-3}, 10^{-1}]$, the recall drops no more than 0.1\% at 50\% MPR in the edge-assisted configuration with 8MHz.
Similarly, the weighted recall drops no more than 0.2\% in the distributed configuration at 50\% MPR with 4MHz.
The fluctuation is relatively small compared to the benefit brought by these exploration techniques.

\section{Conclusion} \label{sect:conclusion}
In this paper, we propose C-MASS, a feature-level CP framework that enables object-oriented scheduling of CoVs under a bandwidth constraint.
The framework requires replay and prediction to maintain an empirical second-order perception topology, which is sufficient to describe the perceptual state of the CP system. 
To solve the variant of the BMC problem with the second-order perception topology approximation, we propose a hybrid greedy scheduling algorithm, with low complexity and a positive worst-case approximation ratio.
The C-MASS scheduling algorithm is adapted from the hybrid greedy algorithm, with additional features to balance exploration and exploitation. 
We have evaluated the C-MASS framework on a realistic SUMO trace using a fitted statistical detection model.
The numerical result demonstrates that, compared to the CPM standard, C-MASS achieves 5.8\% higher recall and 4.2\% higher weighted recall in two CP configurations, which is a notable gain of 3D object detection in the context of automated driving.
The superiority of C-MASS over baseline algorithms originates from its object-oriented nature and effective trade-off between exploration and exploitation.

\appendices
\section{Proof of Lemma \ref{lem:submodular}} \label{app:submodular}
By definition, we need to prove 
    \begin{align} \label{eq:submodular-pending}
        g_t^+(\mathcal{S}\cup\{i\}) - g_t^+(\mathcal{S}) \ge g_t^+(\mathcal{T}\cup\{i\}) - g_t^+(\mathcal{T}),
    \end{align}
    for $\forall \mathcal{S} \subseteq \mathcal{T} \subset \mathcal{V}_t$ and $i \in \mathcal{V}_t \setminus \mathcal{T}$.
    Since the pending utility is consisted of contributions from all objects, we consider the contribution of each object separately, by assuming $\mathcal{O}_t = \{n\}$.
    
    If the object $n\in\mathcal{P}_t(\mathcal{T})$, then $n \in \mathcal{P}_t(\mathcal{T} \cup \{i\})$. 
    Therefore, $g_t^+(\mathcal{S}\cup\{i\}) - g_t^+(\mathcal{S}) \ge g_t^+(\mathcal{T}\cup\{i\}) - g_t^+(\mathcal{T}) = w_n - w_n = 0$, which satisfies (\ref{eq:submodular-pending}).
    Otherwise, $n\notin\mathcal{P}_t(\mathcal{T})$ and $n\notin\mathcal{P}_t(\mathcal{S})$.
    
    If $n \in \mathcal{P}_t(\mathcal{S}\cup\{i\})$, then 
    \begin{align*}
        &[g_t^+(\mathcal{S}\cup\{i\}) - g_t^+(\mathcal{S})] - [g_t^+(\mathcal{T}\cup\{i\}) - g_t^+(\mathcal{T})] \\ 
        \ge &[w_n - g_t^+(\mathcal{T}\cup\{i\})] + [g_t^+(\mathcal{T}) - g_t^+(\mathcal{S})] \ge 0.
    \end{align*}
    
    If $n \notin \mathcal{P}_t(\mathcal{S}\cup\{i\})$, then $n \notin \mathcal{P}_t(\mathcal{T}\cup\{i\})$, which means $i$ cannot jointly detect $n$ with $\forall k \in \mathcal{T}$.
    Therefore,
    \begin{align*}
        &g_t^+(\mathcal{S}\cup\{i\}) - g_t^+(\mathcal{S}) \\ 
        =\ &w_n \max \left\{ \max_{k\in \mathcal{T}^C\setminus\{i\}}\mathds{1}\left\{n\in\mathcal{P}^{(2)}_t(i,k)\right\} \tfrac{B_i(t)}{B_i(t)+B_k(t)} \right. \\
        &- \left. \max_{\substack{j\in\mathcal{S} \\ k\in \mathcal{T}^C\setminus\{i\}}}\mathds{1}\left\{n\in\mathcal{P}^{(2)}_t(j,k)\right\} \tfrac{B_j(t)}{B_j(t)+B_k(t)} ,\ 0 \right\}.
    \end{align*}
    Similarly,
    \begin{align*}
        &g_t^+(\mathcal{T}\cup\{i\}) - g_t^+(\mathcal{T}) \\ 
        =\ &w_n \max \left\{ \max_{k\in \mathcal{T}^C\setminus\{i\}}\mathds{1}\left\{n\in\mathcal{P}^{(2)}_t(i,k)\right\} \tfrac{B_i(t)}{B_i(t)+B_k(t)} \right. \\
        &- \left. \max_{\substack{j\in\mathcal{T} \\ k\in \mathcal{T}^C\setminus\{i\}}}\mathds{1}\left\{n\in\mathcal{P}^{(2)}_t(j,k)\right\} \tfrac{B_j(t)}{B_j(t)+B_k(t)} , \ 0 \right\}.
    \end{align*}
    Because $\mathcal{S} \subseteq \mathcal{T}$, (\ref{eq:submodular-pending}) is satisfied and the proof is finished.

\section{Proof of Lemma \ref{lem:alg1-calc}} \label{app:alg1-calc}
    a) It can be observe that $\mathbf{d}$ and $\mathbf{P}$ do not decrease when they are updated by (\ref{eq:update-level}) and (\ref{eq:update-matrix}).
    
    If an object $n$ satisfies $\exists i \in \mathcal{A},\ n \in \mathcal{P}^{(1)}(i)$, then $\mathbf{P}_{i,n} = 1$ as initialized, and thus $\mathbf{d}_n \gets 1$ when $i$ is scheduled.

    If an object $n$ satisfies $\exists i,j \in \mathcal{A}, \ n \in \mathcal{P}^{(2)}(i,j)$, w.l.o.g. assume $i$ is scheduled later than $j$.
    When $j$ is scheduled, $\mathbf{P}_{i,n} \gets 1$ because $n \in \mathcal{P}^{(2)}(i,j)$. 
    Then $\mathbf{d}_n \gets 1$ when $i$ is scheduled.

    Otherwise, the object $n$ satisfies $\forall i \in \mathcal{A}, n\notin\mathcal{P}^{(1)}(i)$, and $\forall i,j \in \mathcal{A}, \ n \notin \mathcal{P}^{(2)}(i,j)$.
    Then for $\forall i \in \mathcal{A}$, the value of $\mathbf{P}_{i,n}$ remains the same as initialized when the auxiliary matrix is updated.
    Therefore, 
    \begin{align*}
        \mathbf{d}_n = \mathbf{P}_{i,n} = \max_{i\in\mathcal{A}, j\in \mathcal{A}^C}\mathds{1}\left\{n\in\mathcal{P}^{(2)}_t(i,j)\right\} \tfrac{B_i}{B_i+B_j}. 
    \end{align*}
    
    b) Again, the contribution of each object $n$ is considered separately, by assuming $\mathcal{O}_t = \{n\}$. We will compare the results by calculations (\ref{eq:calc-marginal-immediate}), (\ref{eq:calc-marginal-pending}) and the definitions (\ref{eq:marginal-actual}), (\ref{eq:marginal-pending}).
    
    If $n\in\mathcal{P}(\mathcal{A})$, by calculation, $\mathbf{d}_n=1$, $\max\{\mathbf{P}_{i,n} - \mathbf{d}_n,\ 0\} = \max\{\lfloor \mathbf{P}_{i,n} \rfloor - \lfloor\mathbf{d}_n\rfloor,\ 0\} = 0$. 
    By definition, object $n$ contributes $w_n$ to both $g(A)$ and $g(A\cup\{i\})$, and it has no other contribution to $g^+(A)$ and $g^+(A\cup\{i\})$. 
    Therefore, $g(i|\mathcal{A}) = g^+(i|\mathcal{A}) = 0$ in both ways.
    
    If $n\in\mathcal{P}(\mathcal{A}\cup\{i\})$ and $n \notin \mathcal{P}(\mathcal{A})$, by calculation, $\mathbf{P}_{i,n} = 1$, and $\mathbf{d}_n = \max_{j\in\mathcal{A}, k\in \mathcal{A}^C}\mathds{1}\left\{n\in\mathcal{P}^{(2)}_t(j,k)\right\} \tfrac{B_j}{B_j+B_k}$.
    By definition, object $n$ contributes $0$ to $g(A)$, $w_n \max_{j\in\mathcal{A}, k\in \mathcal{A}^C}\mathds{1}\left\{n\in\mathcal{P}^{(2)}_t(j,k)\right\} \tfrac{B_j}{B_j+B_k}$ to $g^+(A)$, and $w_n$ to both $g^+(A\cup\{i\})$ and $g^+(A\cup\{i\})$.
    Therefore, $g(i|\mathcal{A}) = w_n$, and $g^+(i|\mathcal{A}) = w_n (1- \max_{j\in\mathcal{A}, k\in \mathcal{A}^C}\mathds{1}\left\{n\in\mathcal{P}^{(2)}_t(j,k)\right\} \tfrac{B_j}{B_j+B_k})$ in both ways.

    If $n\notin\mathcal{P}(\mathcal{A}\cup\{i\})$,
    \begin{align*}
        \mathbf{P}_{i,n} &= \max_{k\in \mathcal{A}^C\setminus\{i\}}\mathds{1}\left\{n\in\mathcal{P}^{(2)}(i,k)\right\} \tfrac{B_i}{B_i+B_k},
    \end{align*}
    and $\mathbf{d}_n = \max_{j\in\mathcal{A}, k\in \mathcal{A}^C}\mathds{1}\left\{n\in\mathcal{P}^{(2)}_t(j,k)\right\} \tfrac{B_j}{B_j+B_k}$.
    By definition, object $n$ contributes 0 to both $g(A)$ and $g(A\cup\{i\})$, $w_n \max_{j\in\mathcal{A}, k\in \mathcal{A}^C}\mathds{1}\left\{n\in\mathcal{P}^{(2)}_t(j,k)\right\} \tfrac{B_j}{B_j+B_k}$ to $g^+(A)$, and $w_n \max_{j\in\mathcal{A}\cup\{i\}, k\in \mathcal{A}^C\setminus\{i\}}\mathds{1}\left\{n\in\mathcal{P}^{(2)}_t(j,k)\right\} \tfrac{B_j}{B_j+B_k}$ to $g^+(A\cup\{i\})$.
    It is easy to see $g(i|\mathcal{A}) = 0$ in both ways.
    On the other hand, since $i$ cannot jointly detect $n$ with $\forall j \in \mathcal{A}$, 
    \begin{align*}
        g^+(i|\mathcal{A}) &= w_n \left[ \max_{j\in\mathcal{A}\cup\{i\}, k\in \mathcal{A}^C\setminus\{i\}} \mathds{1}\left\{n\in\mathcal{P}^{(2)}_t(j,k)\right\} \tfrac{B_j}{B_j+B_k} \right. \notag \\
        &\quad \left. - \max_{j\in\mathcal{A}, k\in \mathcal{A}^C\setminus\{i\}} \mathds{1}\left\{n\in\mathcal{P}^{(2)}_t(j,k)\right\} \tfrac{B_j}{B_j+B_k} \right] \notag \\
        &= w_n \max\left\{\left[ \max_{k\in \mathcal{A}^C\setminus\{i\}} \mathds{1}\left\{n\in\mathcal{P}^{(2)}_t(i,k)\right\} \tfrac{B_i}{B_i+B_k} \right.  \right. \notag \\
        &\quad \left. \left. - \max_{j\in\mathcal{A}, k\in \mathcal{A}^C\setminus\{i\}} \mathds{1}\left\{n\in\mathcal{P}^{(2)}_t(j,k)\right\} \tfrac{B_j}{B_j+B_k} \right] ,\ 0 \right\} \notag  \\
        &= w_n \max\{\mathbf{P}_{i,n} - \mathbf{d}_n,\ 0\}.
    \end{align*}
    To conclude, each object $n$ has equal contributions to the calculations and the definitions of the marginal utilities.

\section{Proof of Lemma \ref{lem:local}} \label{app:local}
    At iteration $n$, at least $g(OPT) - g(\mathcal{A}_n)$ worth actual utility of objects not fully detected by $S_i$ are detected by the $N$ CoVs in $OPT$.
    Therefore, by the pigeonhole principle, one of the $N$ CoVs in $OPT$ must independently or jointly detect at least $(g(OPT) - g(\mathcal{A}_n))/N$ worth of objects.
    
    If there is no joint detection, i.e., $C = 0$, then $g(\mathcal{S}) = g^+(\mathcal{S}) = h(\mathcal{S}), \forall \mathcal{S} \subseteq \mathcal{V}$.
    That CoV has an actual utility of at least $(g(OPT) - g(\mathcal{A}_n))/N$.
    Then the greedy algorithm must schedule a CoV with at least $\frac{1}{N}(g(OPT) - g(\mathcal{A}_n))$.
    
    When there are non-trivial second-order perception topology, i.e., $C \ge 1$, then $\lambda = \frac{1}{C+1} \in (0, 1/2]$. That CoV should have a hybrid utility of at least $\frac{\lambda}{2N} (g(OPT) - g(\mathcal{A}_n))$.
    Since the algorithm schedules the CoV $i$ with maximum hybrid utility, the scheduled CoV must satisfy $h(i|\mathcal{A}_n) \ge \frac{\lambda}{2N} (g(OPT) - g(\mathcal{A}_n))$.
    We decompose the hybrid utility of the scheduled CoV $i$ by  
    \begin{align*}
        h(i|\mathcal{A}_n) = X + (\lambda/2)Y + (1-\lambda/2) Z,
    \end{align*}
    where $X$ is from independent detections of $i$, $(\lambda/2)Y$ is from initiating joint detections of objects with actual utility $Y$, and $(1-\lambda/2)Z$ is from completing joint detections of objects with actual utility $Z$.
    
    Let $X + Z = \mu h $, the one-step actual utility satisfies (\ref{eq:one-step}) if $\lambda \mu/2 \ge \gamma$.
    Otherwise, $\lambda \mu/2 < \gamma$, and denote $h = h(i|\mathcal{A}_n)$ for simplicity.
    
    (a) If $C = 1$, $\lambda=1/2$, $\gamma=1/6$, then $\mu < 2/3$.
    After scheduling $i$, the only second-order collaborator of $i$ should have at least a hybrid utility of
    $( 1- \lambda/2) Y > ( 1- \lambda/2) (\lambda/2)^{-1} (1-\mu)h > h$, while the hybrid utilities of others are not increased.
    Therefore, it must be scheduled next.
    The total actual utility of the two steps is $X+Y+Z > \mu h + (\lambda/2)^{-1} (1-\mu)h > 2h = \frac{1}{2N} (g(OPT) - g(\mathcal{A}_n))$, which satisfies (\ref{eq:two-step}).

    (b) If $C \ge 2$, $\lambda = \frac{1}{C+1}$, $\gamma=\frac{1}{6C+2}$, then $\mu \le \frac{C+1}{3C+1}$.
    By the pigeonhole principle, there exists a second-order collaborator $j$ with at least a hybrid utility of $ h(j|\mathcal{A}_n\cup\{i\}) = \frac{1}{C} Y = \frac{1}{C}(1-\lambda/2) (\lambda/2)^{-1} (1-\mu) h$ from the pending second-order detection.
    Since $ h(j|\mathcal{A}_n\cup\{i\}) \ge h$, CoV $j$ will be scheduled next if there is no other collaborating CoV with a higher hybrid utility.
    If it is scheduled, it provides an actual utility of $\frac{1}{C} (\lambda/2)^{-1} (1-\mu) h$.
    The two-step utility is $\frac{4(c+1)}{3c+1}$, which satisfies (\ref{eq:two-step}).

    On the other hand, assume another collaborating CoV $k$ has a higher hybrid utility than $j$ after scheduling $i$. 
    We decompose the hybrid utility of CoV $k$ before scheduling $i$ by
    \begin{align}
        h(k|\mathcal{A}_n) &= R + (\lambda/2)T+(1-\lambda/2)T_0 \le h(i|\mathcal{A}_n), \label{eq:part-hybrid} 
    \end{align}
    where $R$ is from independent detections of $k$, $(\lambda/2)T$ is from initiating joint detections of objects with actual utility $T$, and $(1-\lambda/2)T_0$ is from completing joint detections of objects with actual utility $T_0$.
    Similarly, decompose the hybrid utility after scheduling $i$ by
    \begin{align*}
        h(k|\mathcal{A}_n\cup\{i\}) &= R'+(\lambda/2)T'+(1-\lambda/2)T'_0 \\
        &\ge h(j|\mathcal{A}_n\cup\{i\}),
    \end{align*}
    where $R',T',T'_0$ denote the corresponding utilities respectively after scheduling $i$.
    Since the uninitialized joint detections of $k$ can only decrease after scheduling $i$, by (\ref{eq:part-hybrid}),
    \begin{align*}
        (\lambda/2)T' \le (\lambda/2)T \le h.
    \end{align*}
    Therefore, the two-step actual utility is at least
    \begin{align*}
        g(\mathcal{A}_{n+2}) - g(\mathcal{A}_n) &= X+Z+R'+T'_0 \\
        &\ge \mu h + R'+(1-\lambda/2)T'_0\\
        &\ge \mu h + h(j|\mathcal{A}_n\cup\{i\}) - h \\
        &\ge \frac{2(c+1)}{3c+1}.
    \end{align*}
    which satisfies (\ref{eq:two-step}). 

\section{Proof of Theorem \ref{thm}} \label{app:thm}
    It suffices to prove that either 
    \begin{align} \label{eq:former-cond}
        g(\mathcal{A}_n) \ge g(\mathcal{A}_{n-1}) \ge \left[1-\left(1-\frac{\gamma}{N}\right)^{n-1}\right] g(\mathcal{A}^*),
    \end{align}
    or
    \begin{align} \label{eq:latter-cond}
    \begin{cases}
        g(\mathcal{A}_n) \ge \left[1-\left(1-\frac{\gamma}{N}\right)^{n}\right] g(\mathcal{A}^*) \\
        g(\mathcal{A}_{n-1}) \ge \left[1-\left(1-\frac{\gamma}{N}\right)^{n-2}\right] g(\mathcal{A}^*),
    \end{cases}
    \end{align}
    where $n$ is the round index, $n=2,\cdots,N$ and $\mathcal{A}_N = \mathcal{A}_\mathrm{greedy}$.
    We proceed by induction on the round index $n$.
    
    For $n=2$, by Lemma~\ref{lem:local}, either $g(\mathcal{A}_1) \ge g(\mathcal{A}_0) + \frac{\gamma}{N} (g(\mathcal{A}^*) - g(\mathcal{A}_0)) = \frac{\gamma}{N} g(\mathcal{A}^*)$ or $g(\mathcal{A}_2) \ge g(\mathcal{A}_0) + \frac{2\gamma}{N} (g(\mathcal{A}^*) - g(\mathcal{A}_0)) = \frac{2\gamma}{N} g(\mathcal{A}^*)$.
    The former means $g(\mathcal{A}_2) \ge g(\mathcal{A}_1) \ge \frac{\gamma}{N} g(\mathcal{A}^*)$, which satisfies (\ref{eq:former-cond}).
    The latter means $g(\mathcal{A}_2) \ge \frac{2\gamma}{N} g(\mathcal{A}^*)$ and $g(\mathcal{A}_1) \ge 0$, which satisfies (\ref{eq:latter-cond}).
    
    If (\ref{eq:former-cond}) is satisfied for $n-1$, 
    \begin{align*}
        g(\mathcal{A}_{n-1}) \ge g(\mathcal{A}_{n-2}) \ge \left[1-\left(1-\frac{\gamma}{N}\right)^{n-2}\right] g(\mathcal{A}^*),
    \end{align*}
    By Lemma~\ref{lem:local}, if the first condition $g(\mathcal{A}_{n-1}) - g(\mathcal{A}_{n-2}) \ge \frac{\gamma}{N}(g(\mathcal{A}^*) - g(\mathcal{A}_{n-1}))$ is hold, we arrive at (\ref{eq:former-cond}) for $n$.
    If the second condition $g(\mathcal{A}_{n}) - g(\mathcal{A}_{n-2}) \ge \frac{2\gamma}{N}(g(\mathcal{A}^*) - g(\mathcal{A}_{n-2}))$ is hold, we arrive at (\ref{eq:latter-cond}) for $n$.

    If (\ref{eq:latter-cond}) is satisfied for $n-1$, i.e., 
    \begin{align*}
        \begin{cases}
        g(\mathcal{A}_{n-1}) \ge \left[1-\left(1-\frac{\gamma}{N}\right)^{n-1}\right] g(\mathcal{A}^*) \\
        g(\mathcal{A}_{n-2}) \ge \left[1-\left(1-\frac{\gamma}{N}\right)^{n-3}\right] g(\mathcal{A}^*),
    \end{cases}
    \end{align*}
    then (\ref{eq:former-cond}) is naturally satisfied for $n$.

    To conclude, if either (\ref{eq:former-cond}) or (\ref{eq:latter-cond}) is satisfied for $n-1$, then either (\ref{eq:former-cond}) or (\ref{eq:latter-cond}) is satisfied for $n$.
    Since Lemma~\ref{lem:local} applies to $n \in \{0,1, \cdots,N-2\}$, 
    we can get 
    \begin{align*}
        g(\mathcal{A}_N) \ge \left[1-\left(1-\frac{\gamma}{N}\right)^{N-1}\right] g(\mathcal{A}^*).
    \end{align*}
    When $N$ is large, 
    \begin{align*}
        &\quad \lim_{N\to\infty}\left[1-\left(1-\frac{\gamma}{N}\right)^{N-1}\right] \\
        &= 1- \lim_{N\to\infty} \left(1-\frac{\gamma}{N}\right)^{N} \lim_{N\to\infty}\left(1-\frac{\gamma}{N}\right)^{-1} \\
        &= 1-e^{-\gamma}.
    \end{align*}


%




\ifCLASSOPTIONcaptionsoff
  \newpage
\fi

\end{document}